\documentclass{article}

\usepackage{arxiv}

\usepackage[utf8]{inputenc} 
\usepackage[T1]{fontenc}    
\usepackage{hyperref}       
\usepackage{url}            
\usepackage{booktabs}       
\usepackage{amsfonts}       
\usepackage{nicefrac}       
\usepackage{microtype}      
\usepackage{xcolor}         
\usepackage{natbib}
\usepackage{amsmath}
\usepackage{amssymb}
\usepackage{amsthm}
\usepackage{graphicx}
\usepackage{multirow}
\usepackage{wrapfig}

\newtheorem{theorem}{Theorem}[section]
\newtheorem{proposition}[theorem]{Proposition}

\title{Statistical Matching via Schr\"odinger Bridge \\ beyond Conditional Independence}
\author{%
  Eunho Koo\\
  Department of Big Data Convergence \\
  Chonnam National University \\
  \And
    Tongseok Lim
    \\
  Mitch Daniels School of Business \\
  Purdue University \\
   \And
  Jinwon Sohn\thanks{Authors are listed in alphabetical order and contributed equally. This work was conducted primarily while Jinwon Sohn was a Ph.D. student at Purdue University.} 
 \\
  Booth School of Business \\
  University of Chicago \\
  \And
}

\newcommand{\be}{\begin{equation}}
\newcommand{\ee}{\end{equation}}
\newcommand{\bea}{\begin{eqnarray}}
\newcommand{\eea}{\end{eqnarray}}
\newcommand{\barr}{\begin{array}}
\newcommand{\earr}{\end{array}}
\newcommand{\bn}{\begin{enumerate}}
\newcommand{\en}{\end{enumerate}}
\newcommand{\bi}{\begin{itemize}}
\newcommand{\ei}{\end{itemize}}
\newcommand{\bbbm}{\begin{pmatrix}}
\newcommand{\eeem}{\end{pmatrix}}

\newcommand{\nn}{\nonumber}

  \definecolor{darkspringgreen}{rgb}{0.09, 0.45, 0.27} 
 \definecolor{darkgray}{rgb}{0.66, 0.66, 0.66}

\usepackage{doi}

\renewcommand{\shorttitle}{Statistical Matching via Schr\"odinger Bridge}
\hypersetup{
pdftitle={Statistical Matching via Schr\"odinger Bridge beyond Conditional Independence},
pdfauthor={Eunho Koo, Jinwon Sohn, Tongseok Lim},
pdfsubject={stat.ML, stat.ME},
pdfkeywords={Statistical Matching, Schr\"odinger Bridge, Conditional Independence},
}

\begin{document}

\maketitle


\begin{abstract}
Statistical matching combines partially overlapping datasets that share covariates $X$ but observe the target $Y$ and auxiliary variables $Z$ separately. Classical approaches typically invoke the conditional independence assumption (CIA), which makes the problem identifiable but fundamentally implies that the imported auxiliary variable provides no additional predictive power for $Y$ once $X$ is known. To capture this latent $Y$--$Z$ dependence, we propose a novel dependency-aware Schr\"odinger bridge for predictive statistical matching. Our approach couples the two separated databases by tilting the conservative CIA baseline with a transportation-based compatibility cost, recovering an informative joint distribution. The resulting statistical learning framework yields full probabilistic posterior rules for bidirectional imputation. Theoretically, we establish a sufficient condition under which the learned bridge strictly improves over the CIA baseline, alongside an exact joint recovery guarantee in the Gaussian setting under an appropriate cost. Across synthetic benchmarks and real-world datasets (CelebA and Adult), we demonstrate that our dependency-aware completion consistently improves downstream predictive utility, proving especially beneficial in settings like data recoding where the underlying population exhibits strong $Y$--$Z$ dependence.
\end{abstract}


\section{Introduction}
Integrating disparate databases is a ubiquitous technique across many real-world applications. For instance, \citet{lee:etal:25} demonstrate that merging supermarket loyalty data with credit card records reveals grocery shopping habits as a powerful proxy for creditworthiness, significantly improving risk predictions for consumers lacking traditional credit scores. Beyond finance, data integration is pivotal in survey studies \citep{jaus:till:23}, data recoding \citep{gare:omer:22}, combining multi-assay genomics \citep{cai:etal:16}, multimodal learning \citep{li:etal:23}, and socio-economic analyses of income and poverty \citep{sera:tonk:17}.

Statistical matching (SM) becomes essential when these files must be merged but lack unique identification keys, often due to strict privacy constraints. Formally, let $P_{X,Y}$ and $P_{X,Z}$ be the observed marginals of an unknown joint population $P_{X,Y,Z}$ over a random vector $(X,Y,Z)\in \mathcal{X}\times \mathcal{Y}\times \mathcal{Z}$. We consider a recipient database $D^A = \{(X_i^A, Y_i^A)\}_{i=1}^{n_A} \sim P_{X,Y}$ and a donor database $D^B = \{(X_j^B, Z_j^B)\}_{j=1}^{n_B} \sim P_{X,Z}$. Here, the covariates $X$ are shared, but $Y$ is exclusive to $D^A$ and $Z$ is exclusive to $D^B$. 
With a slight abuse of notation, we write $D^A = (X^A, Y^A)$ and $D^B = (X^B, Z^B)$. The goal of SM is to impute the missing auxiliary variables $\hat Z^A$ to construct a synthetic joint dataset $\hat D^A = (X^A, Y^A, \hat Z^A)$ that faithfully approximates the true joint law $P_{X,Y,Z}$. Symmetrically, $D^A$ can act as a donor to complete $\hat D^B=(X^B,\hat Y^B,Z^B)$, yielding a fully merged pooled dataset $\hat D^T = \hat D^A \cup \hat D^B$.

To make this joint-recovery problem tractable, classical SM \citep{rubi:86,aluj:etal:07,endr:augu:16,anno:etal:24} typically relies on the conditional independence assumption (CIA), $Y\perp Z \mid X$, to transfer $Z^B$ to the recipient file. While CIA converts an ill-posed identification problem into a feasible imputation task, it is fundamentally restrictive for downstream applications. Under CIA, $P_{Y|X,Z} = P_{Y|X}$, meaning that the imported variable $\hat{Z}^A$ provides absolutely no additional predictive power for $Y$ once $X$ is known. Consequently, standard CIA-based integration is superfluous if the goal is to capture meaningful $Y$--$Z$ interactions. Dependency-aware data integration therefore requires a principled mechanism to impose structural conditions that recover plausible, informative associations among $(X,Y,Z)$ beyond mere conditional independence.

This work introduces a novel perspective for predictive SM beyond CIA by leveraging the Schr\"odinger bridge mechanism \citep{leon:13,chen:etal:21}. Rooted in entropic optimal transport, we formulate a bridge problem that seeks a joint distribution $Q^{\star}$ whose marginals exactly match the observed $P_{X,Y}$ and $P_{X,Z}$. The synthesized dependency of $Q^{\star}$ is anchored to the conservative CIA product coupling $P_\rho=P_XP_{Y\mid X}P_{Z\mid X}$, but is exponentially tilted by a compatibility cost to recover the latent association between $D^A$ and $D^B$ that minimizes distributional discrepancy. Finally, the learned posterior $Q^\star_{Z \mid X,Y}$ serves as a generative model for imputing $\hat Z^A$ conditionally on $D^A$. Because this approach yields a full probabilistic posterior rather than deterministic point matches, it inherently supports stochastic completion and multiple imputation \citep{rubi:87} for principled uncertainty quantification. Figure~\ref{fig:overview1} depicts the overall pipeline of our methodology.
\begin{figure}[t]
    \centering
    \includegraphics[width=\textwidth]{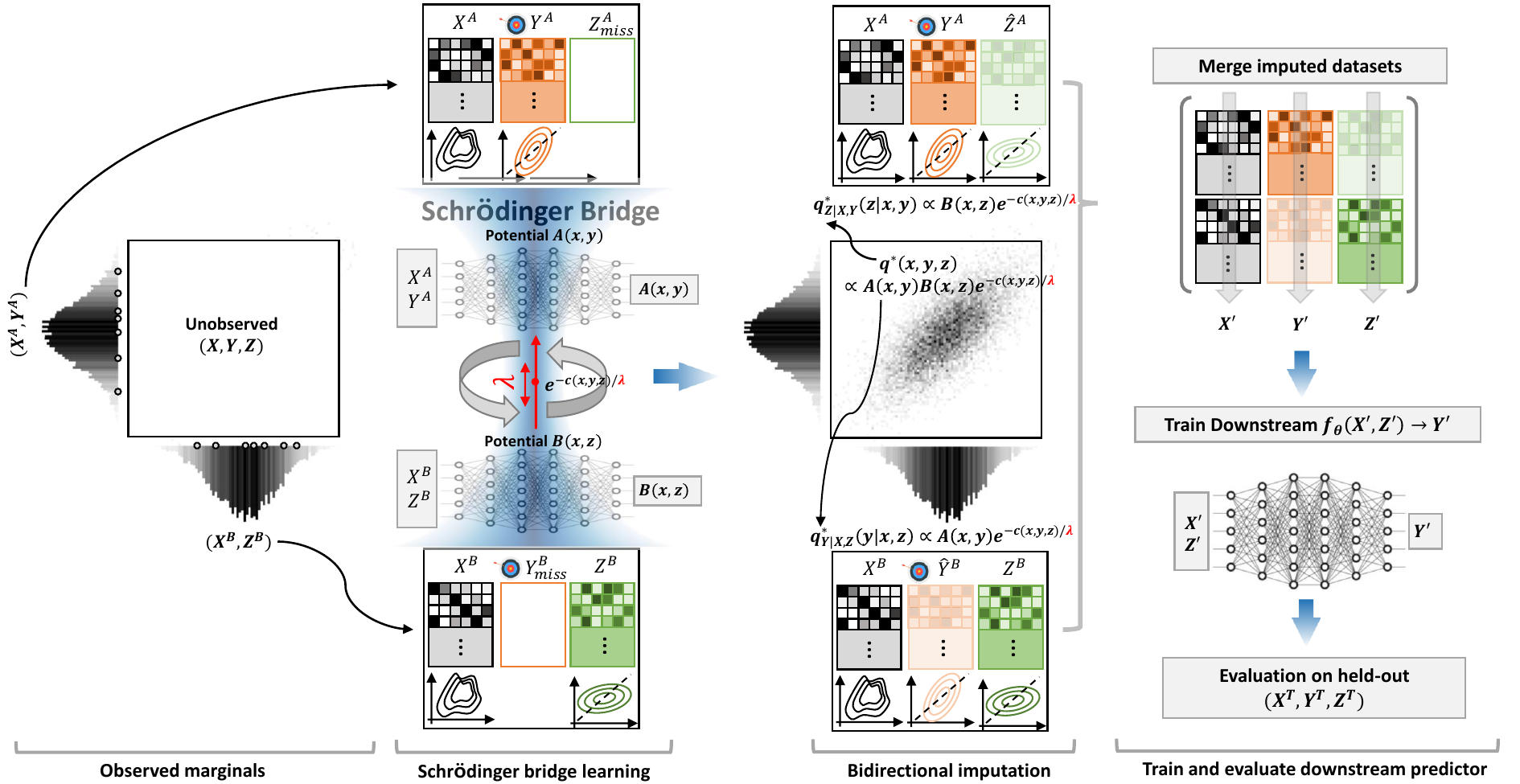}
      \caption{\textbf{Overview of the proposed predictive statistical matching pipeline.} The framework first learns the Schr\"odinger bridge potential functions to link the observed datasets $D^A$ and $D^B$. It then performs bidirectional posterior completion to construct the imputed datasets $\hat D^A$ and $\hat D^B$, which are subsequently merged to train and evaluate the downstream prediction model.}
    \label{fig:overview1}
\end{figure}

\paragraph{Contributions.}
Our contributions are fourfold. First, to our knowledge, this is the first work to cast predictive statistical matching as a dependency-aware Schr\"odinger bridge, providing a principled mechanism to relax the restrictive conditional independence assumption (CIA). Second, we establish a theoretical sufficient condition under which the proposed bridge strictly improves recipient-side posterior completion over the CIA baseline, and we prove that the optimal bridge $Q^\star$ exactly recovers the full joint law under Gaussianity for an appropriate choice of the regularization parameter. Third, we develop a scalable algorithmic framework to estimate the bridge, transforming the potential equations into an efficient statistical risk minimization problem parameterized by neural networks. Finally, across Gaussian, nonlinear synthetic, and real-world benchmark datasets, we empirically demonstrate that our learned couplings can consistently improve downstream predictive utility compared to classical mixed methods and a recent state-of-the-art optimal transport baseline, when a small auxiliary information about the population is available.


\section{Schr\"odinger Bridge for Statistical Matching}
\label{sec:method}

\subsection{Local and global bridge formulations}For any $x\in \mathcal{X}$, let $\rho_x := P_{Y \mid X=x} \otimes P_{Z \mid X=x}$ denote the independent product coupling induced by the CIA. To reconstruct a conditional distribution that captures the latent $Y$--$Z$ dependence beyond $\rho_x$, we propose solving an entropically regularized optimal transport (EOT) problem for each $x \in \mathcal{X}$:
\begin{equation}\label{eq:local_bridge}
\pi_x^\star\in\arg\min_{\pi_x \in \Pi(P_{Y\mid X=x},\, P_{Z\mid X=x})}\left\{\mathbb E_{\pi_x}[c(x,Y,Z)]+\lambda\, \mathrm{KL}(\pi_x \|\rho_x)\right\},
\end{equation}
where $\mathrm{KL}(\mu\|\nu) = \int \log(d\mu/d\nu)d\mu$ is the KL-divergence, and $\Pi(\mu,\nu)$ is the set of joint laws with given marginals $\mu$ and $\nu$. The function $c:\mathcal X\times\mathcal Y\times\mathcal Z\to[0,\infty)$ acts as a compatibility cost, where smaller values indicate stronger latent association between $x$, $y$, and $z$. This objective naturally balances two forces: the first term favors couplings that place mass on compatible $(y,z)$ pairs at the given covariate value $x$, while the second penalizes deviation from the conservative CIA baseline $\rho_x$. Thus, a small $\lambda$ emphasizes compatibility, whereas a large $\lambda$ keeps the solution close to independence. 

To interpret \eqref{eq:local_bridge} from a Schr\"odinger bridge perspective, we define the tilted reference distribution
$dR_x(y,z)
=
\alpha(x)^{-1}
\exp\!\left\{-c(x,y,z)/ \lambda\right\} d\rho_x(y,z),$ 
where $\alpha(x)$ is the normalizing constant. Through this exponential tilt, $R_x$ is no longer a product measure; it encodes conditional dependence between $Y$ and $Z$ while remaining anchored to $\rho_x$. For any coupling $\pi_x\in\Pi(P_{Y\mid X=x},P_{Z\mid X=x})$, 
a standard calculation gives
$\mathrm{KL}(\pi_x\|R_x)
=
\mathrm{KL}(\pi_x\|\rho_x)
+
\frac{1}{\lambda}\,\mathbb E_{\pi_x}[c(x,Y,Z)]
+
\log \alpha(x)$. Hence, minimizing $\mathrm{KL}(\pi_x\|R_x)$ over couplings with fixed marginals is equivalent to solving \eqref{eq:local_bridge}. We refer to \citet{leon:13} and \citet{nutz:21} for a broader background on Schr\"odinger bridges and EOT.

Solving \eqref{eq:local_bridge} for every $x$ is conceptually useful but computationally prohibitive when the covariate space $\mathcal{X}$ is large or continuous, which is common in real-world applications. To address this, we lift the local bridge into a single global formulation. Let $P_\rho(dx,dy,dz)=P_X(dx)P_{Y | X}(dy| 
 x)P_{Z | X}(dz | x)$ denote the full joint law under CIA, and define a single global reference measure:\begin{equation*}
dR(x,y,z)\propto\exp\left\{-c(x,y,z) / \lambda \right\}dP_\rho(x,y,z).
\end{equation*}

We then define the global Schr\"odinger bridge $Q^\star$ as the solution to:\begin{equation}\label{eq:global_bridge}
Q^\star\in\arg\min_{Q\in\Pi(P_{X,Y},\,P_{X,Z})}\mathrm{KL}(Q|R),
\end{equation}
where $\Pi(P_{X,Y},P_{X,Z})$ is the set of joint laws on $(X,Y,Z)$ whose $(X,Y)$- and $(X,Z)$-marginals match the observed data. Because any candidate $Q$ and the reference $R$ share the same marginal $P_X$, a standard decomposition of the KL divergence ensures the local and global formulations are equivalent: $Q^\star$ solves \eqref{eq:global_bridge} if and only if its conditional kernel $Q^\star_{Y,Z\mid X=x}$ solves \eqref{eq:local_bridge} for $P_X$-a.s.. Thus, the global formulation \eqref{eq:global_bridge} provides a single, scalable, and tractable objective for practical implementation.

\subsubsection{Structural comparison to existing SOTA model}
\label{sec:comp-ot}
Our Schr\"odinger bridge formulation recovers, in the unregularized limit, the recent optimal transport approach of \citet{gare:omer:22}, hereafter \textsf{GO}. To synthesize the association between $D^A$ and $D^B$ on a finite support, \textsf{GO} constructs a joint coupling $\gamma^{\star} \in \arg\min_{\gamma\in \Pi(P_{X,Y},P_{X,Z})} \int c(y,z)\, d\gamma$, where the cost $c(y,z) := \mathbb{E}[d(X^A, X^B)\mid Y^A=y, Z^B=z]$ reflects the expected covariate discrepancy. This strategy aligns the datasets by minimizing the prespecified matching loss. Conceptually, our local bridge in \eqref{eq:local_bridge} reduces to \textsf{GO}'s unregularized transport problem as $\lambda \to 0$. However, our entropically regularized formulation offers several distinct structural and computational advantages:

\quad\raisebox{0.33ex}{\tiny$\bullet$} \textbf{Strict Convexity and Stability:} \textsf{GO} requires solving a linear program, which can suffer from instability and non-unique solutions. In contrast, the KL-penalty in \eqref{eq:local_bridge} makes our optimization strictly convex. As established in the EOT literature, this regularization dramatically improves both statistical stability and computational efficiency \citep{cutu:13, peyr:cutu:19}.

\quad\raisebox{0.33ex}{\tiny$\bullet$} \textbf{Universal support:} The standard \textsf{GO} formulation is strictly tailored to discrete settings and relies on tabular linear programming. Because our bridge is formulated through continuous reference measures, it naturally accommodates both continuous and discrete (one-hot encoded) variables, making it highly scalable to complex real-world predictive pipelines.

While the idea of optimal transport in statistical matching has been explored previously \citep{kade:01, jaus:till:23, gare:omer:22}, to our knowledge, this work represents the first attempt to cast statistical matching as a dependency-aware Schr\"odinger bridge.

\subsection{Posterior completion}
\label{sec:posterior_completion}

Assuming the relevant distributions admit densities with respect to an underlying base measure, let $q^\star$, $p_{X,Y}$, and $p_{X,Z}$ denote the densities of $Q^\star$, $P_{X,Y}$, and $P_{X,Z}$, respectively. The solution to the global formulation in \eqref{eq:global_bridge} admits a unique structural representation characterized by two potential functions.
\begin{proposition}\label{prop:global_bridge_form}
The density $q^{\star}$ of the solution $Q^{\star}$ to \eqref{eq:global_bridge} factors as
\begin{equation}\label{eqn:global_joint_sol}q^\star(x,y,z) = A(x,y) B(x,z) e^{-c(x,y,z)/\lambda},
\end{equation}
where the positive potentials $A(x,y)$ and $B(x,z)$ are determined by the system of equations:
\begin{equation}\label{system}
A(x,y)\int B(x,z)e^{-c(x,y,z)/\lambda}dz = p_{X,Y}(x,y), \ \ B(x,z)\int A(x,y)e^{-c(x,y,z)/\lambda}dy = p_{X,Z}(x,z).
\end{equation}
Consequently, the posterior conditional densities for completion are given by
\begin{equation}\label{eqn:global_sol}
q^{\star}_{Z\mid X,Y}(z\mid x,y)\propto B(x,z)e^{-c(x,y,z)/\lambda}, \quad q^{\star}_{Y\mid X,Z}(y\mid x,z)\propto A(x,y)e^{-c(x,y,z)/\lambda}.
\end{equation}
\end{proposition}
This representation directly yields a stochastic completion rule. For each recipient unit $i=1,\dots,n_A$, we impute the missing auxiliary variable by sampling from the learned posterior:
\begin{equation}\label{eq:posterior_z}
\hat Z_i^A \sim q^\star_{Z\mid X,Y}(\cdot \mid X_i^A,Y_i^A)  \propto B(X_i^A,z)e^{-c(X_i^A,Y_i^A,z)/\lambda},
\end{equation}
yielding the completed dataset $\hat D^A$. By drawing from a generative posterior rather than deterministic point matches, each sampling pass yields a distinct completed database. This stochasticity provides the functionality of multiple imputation \citep{rubi:87}, allowing downstream models to properly quantify matching uncertainty. The completion for $\hat D^B$ proceeds symmetrically using $q^\star_{Y\mid X,Z}$.

\subsection{A learning framework to find the potentials}\label{sec:learning_bridge}
We propose a scalable statistical learning framework to compute the bridge potentials. The central idea is to transform the recursive integral equations in Proposition~\ref{prop:global_bridge_form} into a statistical risk minimization problem. Let $A_{p}(x,y)=A(x,y)/p_{X,Y}(x,y)$. As the two observed datasets share the same marginal distribution $P_X$, the ratio of their joint densities simplifies to $d(x,y,z) := p_{X,Y}(x,y)/p_{X,Z}(x,z) = p_{Y\mid X}(y\mid x)/p_{Z\mid X}(z\mid x)$. We can thus rewrite the system \eqref{system} as
\begin{equation}\label{eqn:potential-dynamics}
A_p(x,y)\int B(x,z)e^{-c(x,y,z)/\lambda}dz=1, \ \  B(x,z)\int A_p(x,y)d(x,y,z)e^{-c(x,y,z)/\lambda}dy=1.
\end{equation}
We introduce an auxiliary proposal distribution $q$ (e.g., a simple base measure or empirical marginal) for $Y$ and $Z$. By importance weighting, the conditions in \eqref{eqn:potential-dynamics} are equivalent to:\begin{align*}A_p(x,y)\mathbb{E}_{Z\sim q}\left[\frac{B(x,Z)}{q(Z)}e^{-c(x,y,Z)/\lambda}\right]=1, \ \  B(x,z) \mathbb{E}_{Y\sim q}\left[\frac{A_p(x,Y)}{q(Y)}d(x,Y,z)e^{-c(x,Y,z)/\lambda}\right]=1.\end{align*}We parameterize $A_p$ and $B$ using sufficiently flexible neural networks $A_{p}(x,y;\theta)$ and $B(x,z;\phi)$. {To ensure identifiability, we constrain the potential $B$ to have unit norm, $\| B(\cdot;\phi) \|_{L^2(P_{X,Z})} = 1$, via normalization.} We can now define a statistical risk for learning the optimal parameters $\theta$ and $\phi$:
\begin{align}
\label{eqn:opt-loss}
&L(\theta,\phi) = \mathbb{E}_{P_{X,Y}}\big[\big\{A_{p}(X,Y;\theta)\mathbb{E}_{Z\sim q}\big[\big(B(X,Z;\phi)/q(Z)\big)e^{-c(X,Y,Z)/\lambda}\big]-1\big\}^2\big] \nonumber \\
&+ \mathbb{E}_{P_{X,Z}}\big[\big\{B(X,Z;\phi) \mathbb{E}_{Y\sim q}\big[\big(A_{p}(X,Y;\theta)/q(Y)\big)d(X,Y,Z)e^{-c(X,Y,Z)/\lambda}\big]-1\big\}^2\big].
\end{align}
By construction, $L(\theta,\phi)\geq 0$, and the optimal parameters $(\theta^*,\phi^*)$ that achieve $L(\theta^*,\phi^*)=0$ satisfy the system \eqref{eqn:potential-dynamics} almost surely. In practice, we minimize the empirical estimate of $L(\theta,\phi)$ using alternating gradient descent. {The outer expectations are approximated using mini-batches from $D^A$ and $D^B$, while the inner expectations are estimated via Monte Carlo sampling from $q$. The density ratio $d(x,y,z)$ can be readily evaluated by pre-fitting conditional density estimators using off-the-shelf methods. Finally, the learned potentials $A_p(x,y;\theta^*)$ and $B(x,z;\phi^*)$ are used to execute the posterior completion. For instance, if the support of $Z$ is finite (or empirically approximated by the donor pool $\hat {\mathcal{Z}} = \{Z_j \in D^B\}$), the imputation procedure in \eqref{eq:posterior_z} reduces to sampling from a categorical distribution where the probability of selecting $z_k \in \hat {\mathcal{Z}}$ is proportional to 
$B(x,z_k;\phi^*)e^{-c(x,y,z_k)/\lambda}.$
}

\section{Theoretical properties}
\label{sec:theory}
We now investigate when the bridge-based completion strictly improves upon the CIA baseline, and when it is capable of exact joint recovery. Recall that $P=P_{X,Y,Z}$ denotes the true joint law and $P_\rho=P_X P_{Y\mid X}P_{Z\mid X}$ denotes the product law induced by CIA.
\subsection{A sufficient condition for informative improvement}
For any candidate joint law $Q$, we quantify the recipient-side predictive gain over the CIA baseline by the expected difference in Kullback-Leibler divergences:
\begin{equation}\label{eq:delta_z}
\Delta_Z(Q) = \mathbb E_{P_{X,Y}} \big[ \mathrm{KL}(P_{Z\mid X,Y} \,|\, P_{Z\mid X}) - \mathrm{KL}(P_{Z\mid X,Y} \,|\, Q_{Z\mid X,Y}) \big]. \nn
\end{equation}
A strictly positive $\Delta_Z(Q) > 0$ indicates that $Q_{Z\mid X,Y}$ provides a more accurate posterior for recipient-side completion on average than the conservative baseline $P_{Z\mid X}$.
\begin{theorem}\label{thm:improvement}
Assume the cost satisfies $\mathbb E_{P_\rho}[c(X,Y,Z)^2] < \infty$ and $c \ge 0$. Define the constants
\[
\delta_c = (\mathbb E_{P_\rho} - \mathbb E_{P})[c(X,Y,Z)], \ \  C = \frac{1}{2}\mathbb E_{P_\rho}[c(X,Y,Z)^2], \ \ C_\lambda = \frac{1}{2}\left(\mathbb E_{P_\rho}[e^{-c(X,Y,Z)/\lambda}] - 1\right)^2.\]
Let $Q^\star$ be the solution to the global bridge in \eqref{eq:global_bridge}. Then, for every $\lambda>0$, the informative gain satisfies
\begin{equation}\label{eq:improvement_bound}
\Delta_Z(Q^\star) \ge \frac{\delta_c}{\lambda} - \frac{C}{\lambda^2} + C_\lambda.\end{equation}
\end{theorem}
Theorem~\ref{thm:improvement} identifies the expected cost difference $\delta_c$ as the primary driver of improvement. Specifically, if $\delta_c > 0$---meaning the cost successfully assigns smaller penalties to $(x,y,z)$ triples that are more probable under the true dependence than under the independent baseline---the leading term $\delta_c/\lambda$ dominates, guaranteeing a strictly positive predictive gain $\Delta_Z(Q^\star) > 0$ for appropriately large $\lambda$. By symmetry, an identical conclusion holds for donor-side completion via $Q^\star_{Y\mid X,Z}$.

\subsection{Analysis of full joint recovery}
\label{sec:joint-recovery}
We now analyze the conditions for exact joint recovery, focusing on the case where $P_{X,Y,Z}$ is a multivariate Gaussian distribution with an $x$-independent compatibility cost, i.e., $c(x,y,z)=c(y,z)$.

\begin{theorem}\label{thm:gaussian_recovery}
Suppose $X\in\mathbb R^d$, $Y\in\mathbb R$, and $Z\in\mathbb R$ are jointly Gaussian with law $P=\mathcal N(m,\Sigma)$. Let $\sigma_{YZ\mid X}=\Sigma_{YZ}-\Sigma_{YX}\Sigma_{XX}^{-1}\Sigma_{XZ}$ denote the true conditional covariance between $Y$ and $Z$ given $X$. Consider the cost functions
\begin{equation}\label{eq:gaussian_cost}
c_{\mathrm{align}}(y,z)=\left|\mathbb E[X\mid Y=y]-\mathbb E[X\mid Z=z]\right|^2, \quad \text{and} \quad c_{\mathrm{oracle}}(y,z) = - \mathrm{sgn}(\sigma_{YZ|X}) y z.
\end{equation}
If $c = c_{\mathrm{oracle}}$, or if $c = c_{\mathrm{align}}$ and $\sigma_{YZ\mid X}\,\Sigma_{YX}\Sigma_{XZ}>0$, then there exists a regularization parameter $\lambda^\star >0$ such that the global bridge in \eqref{eq:global_bridge} achieves exact recovery, $Q^\star=P$.
\end{theorem}

This result highlights a fundamental property of our dependency-aware bridge: rather than serving merely as an approximate heuristic, it can perfectly reconstruct the unobserved joint distribution. The oracle cost $c_{\mathrm{oracle}}$ establishes a theoretical guarantee that exact recovery is always possible under Gaussianity. Because this cost depends on the unobserved conditional covariance, it serves primarily as an analytical benchmark. In contrast, the alignment cost $c_{\mathrm{align}}$ is fully empirical and computable directly from the separated marginal datasets. The theorem shows that optimizing the bridge with this observable surrogate still achieves exact recovery, provided the latent residual correlation aligns geometrically with the observed marginal projections ($\sigma_{YZ\mid X}\,\Sigma_{YX}\Sigma_{XZ}>0$). The explicit forms of the optimal parameter $\lambda^\star$ for both costs, along with the proof, are deferred to Appendix~\ref{supp:proof}.

Our recovery analysis provides practical guidance for tuning the hyperparameter $\lambda$. If the data is approximately Gaussian and partial information about the joint distribution $P_{Y,Z}$ (e.g., the true correlation or an auxiliary sample) is available, one can directly compute and apply the theoretical $\lambda^{\star}$. More broadly, if a small validation set of fully observed triplets $(X,Y,Z)\sim P$ is accessible, it can be used to select the $\lambda$ that maximizes downstream predictive utility on the completed data.

Assuming access to such auxiliary information is standard practice in the statistical matching literature. For Gaussian population models, \citet{kade:01} and \citet{mori:sche:01} assume a prior distribution over $\sigma_{yz}$ to guide the matching process. In marketing, \citet{gilu:etal:06} formulated customized dependency structures for binary targets. \citet{kim:etal:15} leveraged instrumental variable assumptions for $Y$ and $Z$ to recover the full joint via fractional imputation. For categorical data, \citet{fosd:etal:16} devised a Bayesian fusion method utilizing an auxiliary data file of $(Y,Z)$ or $(X,Y,Z)$. \citet{cont:etal:16} restricted the joint distribution by imposing logical constraints on the behavior of $Y$ and $Z$ given $X$ (e.g., total expenditure $Y$ cannot exceed linearly transformed income $Z$). \citet{more:etal:23} proposed calibrating a prediction-based matching method to account for distribution shifts, justifying the approach under a Gaussian model. These efforts, alongside the works by \citet{sing:etal:93} and \citet{van:etal:20}, highlight the established necessity of incorporating auxiliary structural information to relax the restrictive CIA paradigm.


\section{Experiments}
\label{sec:simul}
\paragraph{Goal}  We evaluate whether the proposed Schr\"odinger bridge (\textsf{SB}) recovers meaningful latent dependence beyond the CIA baseline on Gaussian, synthetic regression, and real-world datasets. The posterior completion procedure (Section~\ref{sec:posterior_completion}) merges the two marginal datasets into a fully imputed training set $\hat D^T= \hat D^A \cup \hat D^B$, while all downstream performance is measured on a held-out test set $D^T=(X^T,Y^T,Z^T)$ never seen during bridge learning, completion, or downstream training. Our goal is to verify whether recovering the unknown $Y$--$Z$ dependence translates into tangible gains in downstream prediction for $Y$, which indirectly assesses the quality of the recovery. Bridge potentials and downstream predictors are parameterized as MLPs; full details are deferred to Appendix~\ref{app:implementation_details}.

\paragraph{Reference methods}
We evaluate three core learning schemes: \textsf{NAIVE} predicts $Y$ from $X$ alone (the CIA baseline); \textsf{SB} predicts $Y$ using the pooled imputed dataset $\hat D^T$; and \textsf{ORACLE} trains on fully observed triplets $(X,Y,Z)$ to establish an empirical upper bound. We additionally compare against the optimal transport baseline \textsf{GO} and five mixed-method variants, \textsf{MM1}--\textsf{MM5}, which combine parametric models with hot-deck or nearest-neighbor matching \citep{dora:etal:06}. {We report only the best-performing \textsf{MM} variant per setting.} To ensure a fair comparison with \textsf{GO}, which natively requires discrete data, all models are evaluated on binned datasets. This protocol isolates the structural benefits of our EOT formulation over the unregularized OT approach (Section~\ref{sec:comp-ot}), though our bridge inherently accommodates mixed continuous and discrete variables without compromising computational efficiency. Full details regarding the binning procedure are provided in Appendix~\ref{app:discrete_protocol}.

\paragraph{Choice of $\lambda$} Following the statistical matching convention (e.g., \cite{more:etal:23}; Section~\ref{sec:joint-recovery}), we assume a small evaluation subset $C=\{(X_i,Y_i,Z_i)\}_{i=1}^{50}$ is available. \textsf{SB (Weak)} selects $\lambda$ via grid search (10 log-spaced values in $[0.01, 10]$) to maximize downstream performance on $C$, while \textsf{SB (Strong)} uses a large auxiliary set of size $|D^A|=|D^B|=10$k. The two thus represent practical and ideal regimes, respectively: when $\lambda$ selection is insensitive to sample size, their performance should coincide. We denote the chosen values by $\lambda(\text{Weak})$ and $\lambda(\text{Strong})$.

\subsection{Gaussian experiment}
We simulate $(X,Y,Z)$ from a zero-mean trivariate Gaussian distribution, fixing the covariances $\sigma_{XY}=\sigma_{XZ}=0.5$ and varying the latent dependence $\sigma_{YZ} \in \{0.3, 0.5, 0.7, 0.9\}$. We generate 30k instances total, allocated equally across $D^A$, $D^B$, and the fully observed test set $D^T$. We apply the empirical alignment cost $c_{\mathrm{align}}$ in \eqref{eq:gaussian_cost}. Figure~\ref{fig:gaussian_experiments} illustrates how score patterns vary by $\lambda$ against \textsf{NAIVE}. As $\sigma_{YZ}$ increases, the optimal $\lambda$ decreases, confirming that stronger latent $Y$--$Z$ dependence requires more aggressive coupling. When $\sigma_{YZ}\leq 0.5$, \textsf{SB} performs comparably to or slightly below \textsf{NAIVE} in the prediction performance, indicating that minimal useful dependence remains beyond the CIA baseline. However, this pattern shifts dramatically for $\sigma_{YZ} \ge 0.7$. 

The score trend of $\lambda(\text{Weak})$ tracks that of $\lambda(\text{Strong})$ across $\sigma_{YZ}$, validating our methodology even with scarce joint samples. Evaluating the full joint dependence without any $Y$--$Z$ information is fundamentally challenging; when only $\Sigma_{YZ}$ is known \citep{kade:01}, the theoretical $\lambda^\star$ from Theorem~\ref{thm:gaussian_recovery} can be used.

Table~\ref{tab:gaussian_comparison_main} evaluates downstream predictive utility ($R^2$) and donor-side imputation accuracy (RMSE). In informative regimes ($\sigma_{YZ} \in \{0.7, 0.9\}$), \textsf{SB (Strong)} consistently achieves the highest downstream $R^2$, closing $85.6\%$ of the gap to the \textsf{ORACLE} upper bound at $\sigma_{YZ}=0.7$ and approaching $99\%$ at $\sigma_{YZ}=0.9$. \textsf{SB (Weak)} tracks \textsf{SB (Strong)} closely, attaining $80.8\%$ and $98.5\%$ closure respectively. In contrast, \textsf{GO} closes a smaller fraction of the gap ($50.2\%$ and $93.1\%$), and \textsf{Best MM} remains near the \textsf{NAIVE} baseline across all regimes, offering little downstream benefit. On the donor side, \textsf{SB (Weak)} attains the lowest $RMSE^B$ at $\sigma_{YZ}=0.9$ while \textsf{GO} takes the lead at $\sigma_{YZ}=0.7$, among SB (Weak)/GO/Best MM.

We acknowledge that our method can underperform when $Y$--$Z$ dependence is weak and auxiliary tuning data is unavailable. Consequently, our framework is most beneficial when $Y$ and $Z$ are highly correlated, such as in data recoding \citep{gare:omer:22}, where target variables share the same underlying objective but differ in scale or format.

\begin{figure}
    \centering
    \includegraphics[width=0.8\linewidth]{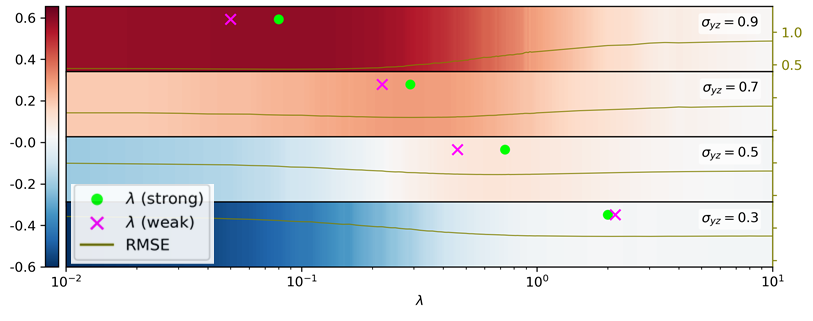}
    \caption{\textbf{(Gaussian experiments)} Heatmaps displaying the downstream $R^2$ improvement ($R_{\textsf{SB}}^2-R_{\textsf{NAIVE}}^2$) across the regularization parameter $\lambda$ (log scale) for varying latent dependence $\sigma_{YZ}$. Darker red indicates greater informative gain over the \textsf{NAIVE} (CIA) baseline. Solid curves track the downstream test RMSE (right axis). Circles and crosses mark the average improvement in $R^2$ out of 10 replicates of $\lambda (\text{Strong}), \lambda (\text{Weak})$, respectively. 
    }
    \label{fig:gaussian_experiments}
\end{figure}

\begin{table}[htbp]
\centering
\caption{(\textbf{Gaussian benchmark}) Down. $R^2$ is the coefficient of determination for predicting $Y$ on the imputed dataset $\hat D^T$; $RMSE^B$ is the imputation error for $\hat Y^B$ in $\hat D^B$ ($\cdot$ if not applicable). Boldface indicates the best performance among \textsf{SB (Weak)}, \textsf{GO}, and the selected \textsf{MM} method; $^\dagger$ marks that \textsf{SB (Strong)} achieves a score at least as good as all three. Reported as mean $\pm$ std.}
\label{tab:gaussian_comparison_main}
\small
\setlength{\tabcolsep}{2.0pt}
\resizebox{\textwidth}{!}{%
\begin{tabular}{cccccccc}
\toprule
$\sigma_{YZ}$ 
& Metric 
& \textsf{NAIVE} 
& \textsf{ORACLE} 
& \textsf{SB (Strong)} 
& \textsf{SB (Weak)}
& \textsf{GO} 
& Best \textsf{MM} \\
\midrule
\multirow{2}{*}{\textbf{0.9}} 
& Down. $R^2$ 
& $0.255 \pm 0.008$
& $0.730 \pm 0.013$
& {$0.725^{\dagger} \pm 0.014$}
& {\bfseries\boldmath $0.723 \pm 0.013$}
& $0.697 \pm 0.019$
& $0.264 \pm 0.009$ \;(\textsf{MM5}) \\
& $RMSE^B$ 
& $\cdot$
& $\cdot$
& $0.572 \pm 0.009$
& {\bfseries\boldmath $0.527 \pm 0.005$}
& $0.558 \pm 0.023$
& $0.866 \pm 0.008$ \;(\textsf{MM2}) \\
\midrule
\multirow{2}{*}{\textbf{0.7}} 
& Down. $R^2$ 
& $0.256 \pm 0.010$
& $0.485 \pm 0.009$
& {$0.452^{\dagger} \pm 0.008$}
& {\bfseries\boldmath $0.441 \pm 0.009$}
& $0.371 \pm 0.064$
& $0.261 \pm 0.008$ \;(\textsf{MM3}) \\
& $RMSE^B$ 
& $\cdot$
& $\cdot$
& $0.896 \pm 0.012$
& $0.872 \pm 0.007$
& {\bfseries\boldmath $0.800 \pm 0.031$}
& $0.866 \pm 0.009$ \;(\textsf{MM2}) \\
\midrule
\multirow{2}{*}{\textbf{0.5}} 
& Down. $R^2$ 
& $0.259 \pm 0.004$
& $0.331 \pm 0.008$
& {$0.265^{\dagger} \pm 0.009$}
& $0.238 \pm 0.016$
& $0.080 \pm 0.097$
& {\bfseries\boldmath $0.260 \pm 0.007$} \;(\textsf{MM2}) \\
& $RMSE^B$ 
& $\cdot$
& $\cdot$
& $1.038 \pm 0.003$
& $1.018 \pm 0.008$
& $0.975 \pm 0.046$
& {\bfseries\boldmath $0.868 \pm 0.010$} \;(\textsf{MM2}) \\
\midrule
\multirow{2}{*}{\textbf{0.3}} 
& Down. $R^2$ 
& $0.260 \pm 0.006$
& $0.262 \pm 0.007$
& $0.207 \pm 0.010$
& $0.209 \pm 0.010$
& $-0.191 \pm 0.095$
& {\bfseries\boldmath $0.260 \pm 0.006$} \;(\textsf{MM2}) \\
& $RMSE^B$ 
& $\cdot$
& $\cdot$
& $1.099 \pm 0.001$
& $1.096 \pm 0.010$
& $1.110 \pm 0.044$
& {\bfseries\boldmath $0.867 \pm 0.009$} \;(\textsf{MM2}) \\
\bottomrule
\end{tabular}
}
\vspace{-0.2cm}
\end{table}

\subsection{Regression experiments}

Next, regression experiments are designed to examine how well \textsf{SB} exploits residual dependence, i.e., the latent association between $Y$ and $Z$ that cannot be explained by $X$ alone. We simulate a standard normal covariate $X=(X_1,X_2)$ and Gaussian noise $(e_1,e_2)$ with variance $\sigma_{\mathrm{noise}}^2\in\{0.1, 0.5\}$ and correlation $\rho_{\mathrm{noise}}\in\{0.3, 0.6, 0.9\}$. We consider the following three data-generating processes:

\quad\raisebox{0.33ex}{\tiny$\bullet$} SYN1 (Linear): $Y=X_1+2X_2+e_1,\quad Z=2X_1+X_2+e_2$.

\quad\raisebox{0.33ex}{\tiny$\bullet$} SYN2 (Nonlinear I): $Y=\sin(X_1)+X_2^2+e_1,\quad Z=\sin(X_1+X_2)+e_2$.

\quad\raisebox{0.33ex}{\tiny$\bullet$} SYN3 (Nonlinear II): $Y=\log(1+e^{X_1})+X_2+e_1,\quad Z=\tfrac12\log(1+e^{X_1})+X_2^2+e_2$.

Because $X$ already explains a substantial portion of both $Y$ and $Z$ in these synthetic scenarios, we employ a cost function that specifically targets the residual structure. We first standardize $y$ and $z$, subtract the empirical conditional mean explained by $X$, and then compare the remaining residuals:
\begin{align*}
    c(x,y,z)=
\Big[
\Big(
\frac{y-\hat\mu_Y}{\hat\sigma_Y}
-{\mathbb E} \Big[\frac{Y-\hat\mu_Y}{\hat\sigma_Y} \Big| X=x\Big]\Big)
-
\Big(
\frac{z-\hat \mu_Z}{\hat\sigma_Z}
-{\mathbb E} \Big[\frac{Z-\hat\mu_Z}{\hat\sigma_Z} \Big| X=x\Big]
\Big)
\Big]^2,
\end{align*}
where $\hat \mu_Y,\hat \mu_Z,\hat\sigma_Y, \hat\sigma_Z$ are the MLE estimates of the Gaussian means and standard deviations, and the conditional mean is estimated by neural networks.
\begin{table}[t]
\centering
\caption{(\textbf{Regression benchmark}) Refer to the exposition in Table~\ref{tab:gaussian_comparison_main}.
}

\label{tab:syn1_comparison_revised}
\small
SYN1 benchmark
\setlength{\tabcolsep}{2.0pt}
\resizebox{\textwidth}{!}{%
\begin{tabular}{ccccccccc}
\toprule
$\sigma^2_{\text{noise}}$ 
& $\rho_{\text{noise}}$ 
& Metric 
& \textsf{NAIVE} 
& \textsf{ORACLE} 
& \textsf{SB (Strong)} 
& \textsf{SB (Weak)}
& \textsf{GO} 
& Best \textsf{MM} \\
\midrule

\multirow{2}{*}{\textbf{0.5}}
& \multirow{2}{*}{\textbf{0.9}}
& Down. $R^2$
& $0.907 \pm 0.001$
& $0.948 \pm 0.004$
& {$0.948^{\dagger} \pm 0.004$}
& {\bfseries\boldmath $0.944 \pm 0.004$}
& $0.776 \pm 0.027$
& $0.908 \pm 0.001$ \;(\textsf{MM2}) \\
& 
& $RMSE^B$
& $\cdot$
& $\cdot$
& {$0.527^{\dagger} \pm 0.002$}
& {\bfseries\boldmath $0.661 \pm 0.005$}
& $1.310 \pm 0.092$
& $0.706 \pm 0.007$ \;(\textsf{MM2}) \\

\midrule

\multirow{2}{*}{\textbf{0.5}}
& \multirow{2}{*}{\textbf{0.6}}
& Down. $R^2$
& $0.907 \pm 0.001$
& $0.925 \pm 0.002$
& {$0.924^{\dagger} \pm 0.001$}
& {\bfseries\boldmath $0.922 \pm 0.001$}
& $0.743 \pm 0.020$
& $0.907 \pm 0.001$ \;(\textsf{MM4}) \\
& 
& $RMSE^B$
& $\cdot$
& $\cdot$
& $0.760 \pm 0.003$
& $0.840 \pm 0.001$
& $1.377 \pm 0.082$
& {\bfseries\boldmath $0.705 \pm 0.006$} \;(\textsf{MM2}) \\

\midrule

\multirow{2}{*}{\textbf{0.5}}
& \multirow{2}{*}{\textbf{0.3}}
& Down. $R^2$
& $0.907 \pm 0.001$
& $0.911 \pm 0.001$
& $0.909 \pm 0.001$
& {\bfseries\boldmath $0.910 \pm 0.001$}
& $0.713 \pm 0.021$
& $0.907 \pm 0.001$ \;(\textsf{MM2}) \\
& 
& $RMSE^B$
& $\cdot$
& $\cdot$
& $0.903 \pm 0.008$
& $0.933 \pm 0.006$
& $1.446 \pm 0.068$
& {\bfseries\boldmath $0.704 \pm 0.005$} \;(\textsf{MM2}) \\

\midrule

\multirow{2}{*}{\textbf{0.1}}
& \multirow{2}{*}{\textbf{0.9}}
& Down. $R^2$
& $0.980 \pm 0.000$
& $0.985 \pm 0.001$
& {$0.985^{\dagger} \pm 0.001$}
& {\bfseries\boldmath $0.984 \pm 0.001$}
& $0.852 \pm 0.018$
& $0.980 \pm 0.000$ \;(\textsf{MM2}) \\
& 
& $RMSE^B$
& $\cdot$
& $\cdot$
& {$0.277^{\dagger} \pm 0.001$}
& {\bfseries\boldmath $0.298 \pm 0.003$}
& $1.132 \pm 0.084$
& $0.316 \pm 0.003$ \;(\textsf{MM2}) \\

\midrule

\multirow{2}{*}{\textbf{0.1}}
& \multirow{2}{*}{\textbf{0.6}}
& Down. $R^2$
& $0.980 \pm 0.000$
& $0.982 \pm 0.000$
& {$0.982^{\dagger} \pm 0.000$}
& {\bfseries\boldmath $0.981 \pm 0.000$}
& $0.850 \pm 0.022$
& $0.980 \pm 0.000$ \;(\textsf{MM2}) \\
& 
& $RMSE^B$
& $\cdot$
& $\cdot$
& $0.336 \pm 0.001$
& $0.376 \pm 0.001$
& $1.143 \pm 0.077$
& {\bfseries\boldmath $0.315 \pm 0.003$} \;(\textsf{MM2}) \\

\midrule

\multirow{2}{*}{\textbf{0.1}}
& \multirow{2}{*}{\textbf{0.3}}
& Down. $R^2$
& $0.980 \pm 0.000$
& $0.980 \pm 0.000$
& {$0.980^{\dagger} \pm 0.000$}
& {\bfseries\boldmath $0.980 \pm 0.000$}
& $0.849 \pm 0.017$
& {\bfseries\boldmath $0.980 \pm 0.000$} \;(\textsf{MM2}) \\
& 
& $RMSE^B$
& $\cdot$
& $\cdot$
& $0.403 \pm 0.004$
& $0.418 \pm 0.002$
& $1.151 \pm 0.073$
& {\bfseries\boldmath $0.315 \pm 0.002$} \;(\textsf{MM2}) \\

\bottomrule
\end{tabular}
}

\vspace{2mm}

SYN2 benchmark
\label{tab:syn2_comparison_revised}
\small
\setlength{\tabcolsep}{2.0pt}
\resizebox{\textwidth}{!}{%
\begin{tabular}{ccccccccc}
\toprule
$\sigma^2_{\text{noise}}$ 
& $\rho_{\text{noise}}$ 
& Metric 
& \textsf{NAIVE} 
& \textsf{ORACLE} 
& \textsf{SB (Strong)} 
& \textsf{SB (Weak)}
& \textsf{GO} 
& Best \textsf{MM} \\
\midrule

\multirow{2}{*}{\textbf{0.5}}
& \multirow{2}{*}{\textbf{0.9}}
& Down. $R^2$
& $0.821 \pm 0.004$
& $0.937 \pm 0.005$
& $0.930 \pm 0.005$
& {\bfseries\boldmath $0.930 \pm 0.005$}
& $0.361 \pm 0.042$
& $0.657 \pm 0.013$ \;(\textsf{MM1}) \\
&
& $RMSE^B$
& $\cdot$
& $\cdot$
& {$0.568^{\dagger} \pm 0.005$}
& {\bfseries\boldmath $0.585 \pm 0.006$}
& $1.497 \pm 0.055$
& $1.588 \pm 0.025$ \;(\textsf{MM2}) \\

\midrule

\multirow{2}{*}{\textbf{0.5}}
& \multirow{2}{*}{\textbf{0.6}}
& Down. $R^2$
& $0.821 \pm 0.003$
& $0.872 \pm 0.002$
& {$0.861^{\dagger} \pm 0.003$}
& {\bfseries\boldmath $0.854 \pm 0.001$}
& $0.232 \pm 0.043$
& $0.641 \pm 0.008$ \;(\textsf{MM1}) \\
&
& $RMSE^B$
& $\cdot$
& $\cdot$
& {$0.774^{\dagger} \pm 0.007$}
& {\bfseries\boldmath $0.865 \pm 0.008$}
& $1.585 \pm 0.055$
& $1.585 \pm 0.024$ \;(\textsf{MM2}) \\

\midrule

\multirow{2}{*}{\textbf{0.5}}
& \multirow{2}{*}{\textbf{0.3}}
& Down. $R^2$
& $0.821 \pm 0.003$
& $0.833 \pm 0.002$
& $0.821 \pm 0.002$
& {\bfseries\boldmath $0.821 \pm 0.003$}
& $0.148 \pm 0.084$
& $0.647 \pm 0.010$ \;(\textsf{MM1}) \\
&
& $RMSE^B$
& $\cdot$
& $\cdot$
& {$0.899^{\dagger} \pm 0.004$}
& {\bfseries\boldmath $0.936 \pm 0.009$}
& $1.656 \pm 0.068$
& $1.584 \pm 0.023$ \;(\textsf{MM2}) \\

\midrule

\multirow{2}{*}{\textbf{0.1}}
& \multirow{2}{*}{\textbf{0.9}}
& Down. $R^2$
& $0.958 \pm 0.001$
& $0.980 \pm 0.002$
& {$0.978^{\dagger} \pm 0.002$}
& {\bfseries\boldmath $0.969 \pm 0.001$}
& $0.445 \pm 0.040$
& $0.759 \pm 0.008$ \;(\textsf{MM1}) \\
&
& $RMSE^B$
& $\cdot$
& $\cdot$
& {$0.285^{\dagger} \pm 0.001$}
& {\bfseries\boldmath $0.309 \pm 0.004$}
& $1.333 \pm 0.061$
& $1.457 \pm 0.024$ \;(\textsf{MM2}) \\

\midrule

\multirow{2}{*}{\textbf{0.1}}
& \multirow{2}{*}{\textbf{0.6}}
& Down. $R^2$
& $0.958 \pm 0.001$
& $0.967 \pm 0.001$
& {$0.966^{\dagger} \pm 0.001$}
& {\bfseries\boldmath $0.964 \pm 0.001$}
& $0.452 \pm 0.019$
& $0.757 \pm 0.008$ \;(\textsf{MM1}) \\
&
& $RMSE^B$
& $\cdot$
& $\cdot$
& {$0.370^{\dagger} \pm 0.002$}
& {\bfseries\boldmath $0.400 \pm 0.001$}
& $1.319 \pm 0.049$
& $1.456 \pm 0.024$ \;(\textsf{MM2}) \\

\midrule

\multirow{2}{*}{\textbf{0.1}}
& \multirow{2}{*}{\textbf{0.3}}
& Down. $R^2$
& $0.958 \pm 0.001$
& $0.960 \pm 0.001$
& {$0.958^{\dagger} \pm 0.000$}
& {\bfseries\boldmath $0.958 \pm 0.001$}
& $0.417 \pm 0.042$
& $0.756 \pm 0.006$ \;(\textsf{MM1}) \\
&
& $RMSE^B$
& $\cdot$
& $\cdot$
& {$0.420^{\dagger} \pm 0.004$}
& {\bfseries\boldmath $0.432 \pm 0.002$}
& $1.358 \pm 0.064$
& $1.455 \pm 0.023$ \;(\textsf{MM2}) \\

\bottomrule
\end{tabular}
}

\vspace{2mm}

SYN3 benchmark
\label{tab:syn3_comparison_revised}
\small
\setlength{\tabcolsep}{2.0pt}
\resizebox{\textwidth}{!}{%
\begin{tabular}{ccccccccc}
\toprule
$\sigma^2_{\text{noise}}$ 
& $\rho_{\text{noise}}$ 
& Metric 
& \textsf{NAIVE} 
& \textsf{ORACLE} 
& \textsf{SB (Strong)} 
& \textsf{SB (Weak)}
& \textsf{GO} 
& Best \textsf{MM} \\
\midrule

\multirow{2}{*}{\textbf{0.5}}
& \multirow{2}{*}{\textbf{0.9}}
& Down. $R^2$
& $0.713 \pm 0.002$
& $0.881 \pm 0.007$
& {$0.879^{\dagger} \pm 0.007$}
& {\bfseries\boldmath $0.875 \pm 0.007$}
& $0.598 \pm 0.042$
& $0.716 \pm 0.010$ \;(\textsf{MM3}) \\
&
& $RMSE^B$
& $\cdot$
& $\cdot$
& {$0.455^{\dagger} \pm 0.001$}
& {\bfseries\boldmath $0.547 \pm 0.007$}
& $0.882 \pm 0.044$
& $0.721 \pm 0.007$ \;(\textsf{MM2}) \\

\midrule

\multirow{2}{*}{\textbf{0.5}}
& \multirow{2}{*}{\textbf{0.6}}
& Down. $R^2$
& $0.713 \pm 0.003$
& $0.786 \pm 0.004$
& {$0.780^{\dagger} \pm 0.004$}
& {\bfseries\boldmath $0.766 \pm 0.004$}
& $0.487 \pm 0.061$
& $0.710 \pm 0.003$ \;(\textsf{MM4}) \\
&
& $RMSE^B$
& $\cdot$
& $\cdot$
& {$0.707^{\dagger} \pm 0.004$}
& $0.830 \pm 0.009$
& $0.980 \pm 0.060$
& {\bfseries\boldmath $0.719 \pm 0.006$} \;(\textsf{MM2}) \\

\midrule

\multirow{2}{*}{\textbf{0.5}}
& \multirow{2}{*}{\textbf{0.3}}
& Down. $R^2$
& $0.712 \pm 0.003$
& $0.730 \pm 0.003$
& $0.716 \pm 0.004$
& {\bfseries\boldmath $0.718 \pm 0.003$}
& $0.381 \pm 0.059$
& $0.710 \pm 0.003$ \;(\textsf{MM3}) \\
&
& $RMSE^B$
& $\cdot$
& $\cdot$
& $0.867 \pm 0.007$
& $0.915 \pm 0.003$
& $1.079 \pm 0.045$
& {\bfseries\boldmath $0.718 \pm 0.005$} \;(\textsf{MM2}) \\

\midrule

\multirow{2}{*}{\textbf{0.1}}
& \multirow{2}{*}{\textbf{0.9}}
& Down. $R^2$
& $0.924 \pm 0.002$
& $0.959 \pm 0.004$
& {$0.957^{\dagger} \pm 0.004$}
& {\bfseries\boldmath $0.956 \pm 0.004$}
& $0.766 \pm 0.036$
& $0.919 \pm 0.004$ \;(\textsf{MM1}) \\
&
& $RMSE^B$
& $\cdot$
& $\cdot$
& {$0.235^{\dagger} \pm 0.005$}
& {\bfseries\boldmath $0.251 \pm 0.002$}
& $0.653 \pm 0.046$
& $0.348 \pm 0.004$ \;(\textsf{MM2}) \\

\midrule

\multirow{2}{*}{\textbf{0.1}}
& \multirow{2}{*}{\textbf{0.6}}
& Down. $R^2$
& $0.924 \pm 0.002$
& $0.939 \pm 0.004$
& {$0.938^{\dagger} \pm 0.004$}
& {\bfseries\boldmath $0.938 \pm 0.002$}
& $0.746 \pm 0.036$
& $0.918 \pm 0.004$ \;(\textsf{MM3}) \\
&
& $RMSE^B$
& $\cdot$
& $\cdot$
& {$0.329^{\dagger} \pm 0.004$}
& {\bfseries\boldmath $0.339 \pm 0.000$}
& $0.676 \pm 0.046$
& $0.347 \pm 0.004$ \;(\textsf{MM2}) \\

\midrule

\multirow{2}{*}{\textbf{0.1}}
& \multirow{2}{*}{\textbf{0.3}}
& Down. $R^2$
& $0.924 \pm 0.002$
& $0.928 \pm 0.004$
& {$0.926^{\dagger} \pm 0.004$}
& {\bfseries\boldmath $0.926 \pm 0.001$}
& $0.721 \pm 0.046$
& $0.918 \pm 0.004$ \;(\textsf{MM3}) \\
&
& $RMSE^B$
& $\cdot$
& $\cdot$
& $0.401 \pm 0.004$
& $0.417 \pm 0.002$
& $0.701 \pm 0.046$
& {\bfseries\boldmath $0.346 \pm 0.004$} \;(\textsf{MM2}) \\

\bottomrule
\end{tabular}
}
\end{table}

Table~\ref{tab:syn1_comparison_revised} reports results across three regression scenarios. The advantage of \textsf{SB (Weak)} over \textsf{NAIVE} grows with the residual dependence $\rho_{\mathrm{noise}}$ and is largest in the nonlinear settings, where the covariate $X$ explains less of the joint structure. In SYN1 (linear), \textsf{NAIVE} is already a strong baseline, yet \textsf{SB (Weak)} still recovers most of the remaining gap to \textsf{ORACLE} once $\rho_{\mathrm{noise}} \geq 0.6$ (e.g., $0.922$ vs.\ $0.925$ at $\sigma^2_{\mathrm{noise}}=0.5,\rho_{\mathrm{noise}}=0.6$). In SYN2, where the dependence is harder to capture, \textsf{SB (Weak)} closes a substantial portion of the gap at $\rho_{\mathrm{noise}}=0.9$ ($0.930$ vs.\ \textsf{ORACLE}\ $0.937$, against \textsf{NAIVE}\ $0.821$), and similarly clear gains appear in SYN3. Throughout,  \textsf{GO} stays well below NAIVE;  \textsf{Best MM} tracks NAIVE closely on SYN1 and SYN3 but falls well below it on SYN2, indicating that neither competitor reliably exploits the residual signal that SB captures. 
When the residual dependence is weak ($\rho_{\mathrm{noise}}=0.3$), \textsf{SB (Weak)} performs comparably to \textsf{NAIVE}, reflecting that little additional predictive information remains to import.

\subsection{Real-data experiments}
\label{sec:real-data-exp}
We evaluate \textsf{SB} on two real-world datasets: CelebA \citep{liu2015deep} and Adult \citep{becker1996adult}. For both scenarios, we employ the empirical alignment cost $c_{\mathrm{align}}$ defined in \eqref{eq:gaussian_cost}. For CelebA, we utilize an 8-dimensional VAE latent embedding of $(218,178,3)$ images \citep{kingma2013auto} as the covariates $X$, the binary \texttt{Attractive} attribute as the target $Y$, and {five binary facial attributes to form a 32-state discrete auxiliary variable $Z$.}
\begin{wrapfigure}{r}{0.59\textwidth}
\vspace{-4mm}
\centering
\includegraphics[width=0.48\textwidth]{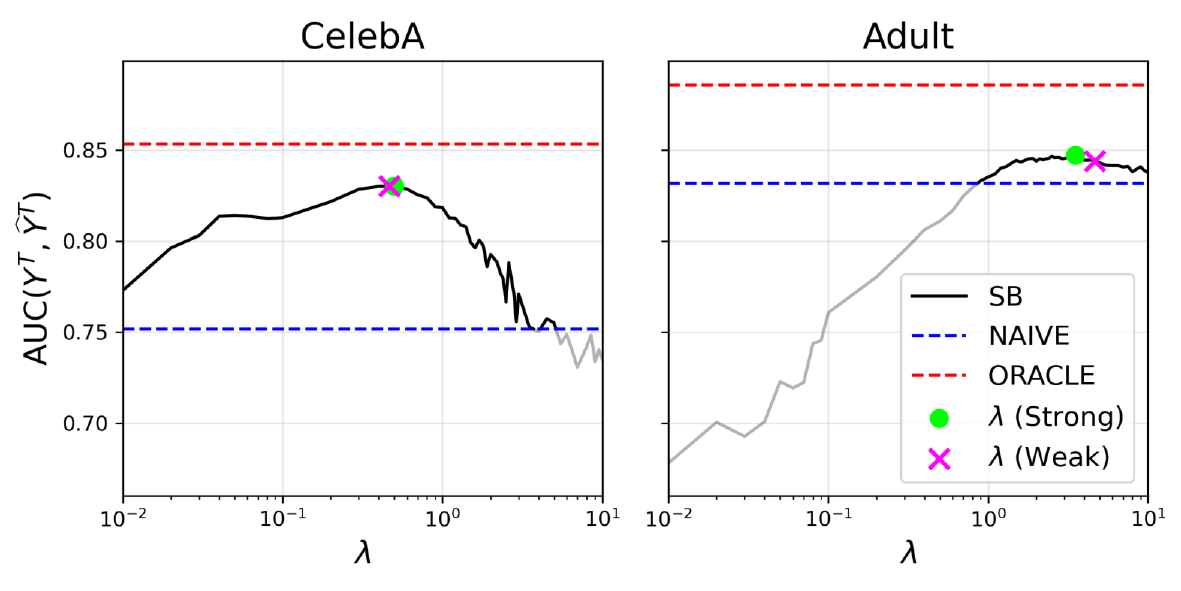}
\caption{Downstream test AUC, $\mathrm{AUC}(Y^T,\hat Y^T)$, vs.\ regularization parameter $\lambda$ on CelebA (left) and Adult (right). Dashed lines mark \textsf{NAIVE} (blue) and \textsf{ORACLE} (red); the green dot and magenta cross indicate $\lambda$ selected by \textsf{SB (Strong)} and \textsf{SB (Weak)}, respectively.}
\label{fig:syn_rec}
\vspace{-2mm}
\end{wrapfigure}
For Adult, $Y$ indicates whether income exceeds \$50K, $X$ comprises standard demographic covariates, and $Z$ is constructed via $K$-means clustering ($K=32$) over education, occupation, workclass, and financial features.

Consistent with the synthetic results, \textsf{SB} delivers tangible downstream gains on both real-world datasets. On CelebA, as shown in Table~\ref{tab:real_comparison_revised}, \textsf{SB (Weak)} attains the best downstream and donor-side AUC, closing $77.2\%$ and $74.5\%$ of the gap to \textsf{ORACLE}, while \textsf{GO} falls below \textsf{NAIVE} and \textsf{Best MM} trails \textsf{SB}. On Adult, where \textsf{NAIVE} is already strong ($0.834$), the headroom is small, yet \textsf{SB (Weak)} still leads on both metrics, with a clear donor-side margin ($0.843$ vs.\ $0.813$). Figure~\ref{fig:syn_rec} shows that the $\lambda$ chosen by \textsf{SB (Weak)} from marginal constraints alone lies near the curve's peak, essentially matching \textsf{SB (Strong)}. Practitioners can thus deploy our framework using only the completed datasets, validating against the \textsf{NAIVE} baseline.
\begin{table}[htbp]
\centering
\caption{\textbf{(Real-world benchmarks)} Down. AUC is the downstream test AUC for predicting $Y$; $AUC^B$ is the donor-side imputation AUC for $\hat Y^B$. Note Table~\ref{tab:gaussian_comparison_main} explains the symbols.} 
\label{tab:real_comparison_revised}
\small
\setlength{\tabcolsep}{2.0pt}
\resizebox{\textwidth}{!}{%
\begin{tabular}{cccccccc}
\toprule
Dataset 
& Metric 
& \textsf{NAIVE} 
& \textsf{ORACLE} 
& \textsf{SB (Strong)} 
& \textsf{SB (Weak)}
& \textsf{GO} 
& Best \textsf{MM} \\
\midrule
\multirow{2}{*}{\textbf{CelebA}}
& Down. AUC
& $0.752 \pm 0.001$
& $0.853 \pm 0.002$
& {$0.830^{\dagger} \pm 0.002$}
& {\bfseries\boldmath $0.830 \pm 0.001$}
& $0.713 \pm 0.020$
& $0.814 \pm 0.007$ \;(\textsf{MM3}) \\
& $AUC^B$
& $\cdot$
& $\cdot$
& {$0.832^{\dagger} \pm 0.002$}
& {\bfseries\boldmath $0.832 \pm 0.002$}
& $0.714 \pm 0.006$
& $0.788 \pm 0.006$ \;(\textsf{MM5}) \\
\midrule
\multirow{2}{*}{\textbf{Adult}}
& Down. AUC
& $0.834 \pm 0.002$
& $0.888 \pm 0.002$
& {$0.845^{\dagger} \pm 0.002$}
& {\bfseries\boldmath $0.844 \pm 0.004$}
& $0.516 \pm 0.039$
& $0.841 \pm 0.003$ \;(\textsf{MM3}) \\
& $AUC^B$
& $\cdot$
& $\cdot$
& {$0.843^{\dagger} \pm 0.001$}
& {\bfseries\boldmath $0.843 \pm 0.001$}
& $0.507 \pm 0.015$
& $0.813 \pm 0.004$ \;(\textsf{MM1}) \\
\bottomrule
\end{tabular}
}
\vspace{-0.2cm}
\end{table}

\section{Conclusion}
\label{sec:conclusion}

We introduced a dependency-aware Schr\"odinger bridge for predictive statistical matching. By exponentially tilting the conditional independence baseline with a compatibility cost, our framework preserves observed marginals while actively synthesizing informative latent dependence. Theoretical results guarantee exact joint recovery under Gaussianity, and empirical benchmarks confirm \textsf{SB} maximizes downstream predictive utility when variables exhibit strong underlying correlation. 

{\bf Broader impacts and limitations.} While dependency-aware matching significantly improves predictive integration, imposing data dependencies risks amplifying historical biases or spurious correlations present in the marginals, necessitating fairness audits when integrating sensitive attributes. Methodologically, our exact recovery guarantees are restricted to Gaussian regimes, and performance depends on the compatibility cost and regularization parameter. While our empirical tuning heuristic performs well, developing a principled, assumption-free tuning criterion remains an open challenge. Future work will explore non-Gaussian recovery, adaptive costs, and algorithmic bias mitigation.


\bibliographystyle{apalike}
\bibliography{ref.bib}

\newpage

\appendix

\bigskip
\begin{center}
{\large\bf Appendix}
\end{center}

\begin{center}
    {\bf Title: Statistical Matching via Schr\"odinger Bridge beyond Conditional Independence}
\end{center}

\section{Proofs}
\label{supp:proof}

\subsection{Proposition~\ref{prop:global_bridge_form}}

\begin{proof}
Let $q(x,y,z)$ and $r(x,y,z)$ denote the densities of $Q$ and $R$, respectively. The  problem \eqref{eq:global_bridge} seeks to minimize $\mathrm{KL}(Q\|R)$ subject to the marginal constraints $\int q(x,y,z)dz = p_{X,Y}(x,y)$ and $\int q(x,y,z)dy = p_{X,Z}(x,z)$. Introducing Lagrange multipliers $f(x,y)$ and $g(x,z)$ for these constraints, the Lagrangian is given by:
\begin{align*}
\mathcal{L}(q, f, g) = \int q(x,y,z)\log \frac{q(x,y,z)}{r(x,y,z)}dxdydz &+ \int f(x,y)\Big(p_{X,Y}(x,y)-\int q(x,y,z)dz\Big)dxdy \\
&+ \int g(x,z)\Big(p_{X,Z}(x,z)-\int q(x,y,z)dy\Big)dxdz.
\end{align*}

Taking the variational derivative with respect to $q$ and setting it to zero yields the optimality condition:
$$
\log \frac{q(x,y,z)}{r(x,y,z)} - f(x,y) - g(x,z) + 1 = 0,
$$
which implies $q^\star(x,y,z) = r(x,y,z)\exp(f(x,y) + g(x,z) - 1)$. 

By the definition of the reference measure $R$, $r(x,y,z) \propto e^{-c(x,y,z)/\lambda}p_{X}(x)p_{Y\mid X}(y | x)p_{Z | X}(z | x)$. Absorbing the terms depending only on $(x,y)$ into $A(x,y) \propto \exp(f(x,y))$ and those depending only on $(x,z)$ into $B(x,z) \propto \exp(g(x,z))$, we obtain the factorized form:
$$
q^\star(x,y,z) = A(x,y)B(x,z)e^{-c(x,y,z)/\lambda}.
$$

The integral equations \eqref{system} follow by substituting this into the marginal constraints, and the conditional distributions follow by dividing $q^\star(x,y,z)$ by the respective marginals $p_{X,Y}$ and $p_{X,Z}$.
\end{proof}

\subsection{Theorem~\ref{thm:improvement}}

\begin{proof}
Recall $P=P_{X,Y,Z}$ is the true joint law, $P_{\rho}=P_{X}P_{Y\mid X}P_{Z\mid X}$ is the CIA product law, and $R$ is defined by $dR(x,y,z) = C_R^{-1} \exp\{-c(x,y,z)/\lambda\} dP_\rho(x,y,z)$, where $C_R = \mathbb{E}_{P_\rho}[\exp\{-c(X,Y,Z)/\lambda\}]$. First, observe that for any joint distribution $Q$ sharing the same $(X,Y)$ and $(X,Z)$ marginals as $P$, the difference in KL divergences simplifies to:
$$
\mathrm{KL}(P\|P_\rho) - \mathrm{KL}(P\|Q) = \mathbb{E}_{P_{X,Y}}[\mathrm{KL}(P_{Z\mid X,Y}\|P_{Z\mid X})] - \mathbb{E}_{P_{X,Y}}[\mathrm{KL}(P_{Z\mid X,Y}\|Q_{Z\mid X,Y})] = \Delta_Z(Q).
$$

Since $Q^\star$ is the KL-projection of $R$ onto $\Pi(P_{X,Y}, P_{X,Z})$, it satisfies the generalized Pythagorean theorem: 
$\mathrm{KL}(P\|R) \geq \mathrm{KL}(P\|Q^\star) + \mathrm{KL}(Q^\star \|R)$, hence $\mathrm{KL}(P\|Q^\star) \leq \mathrm{KL}(P\|R)$. This implies: 
$$
\Delta_Z(Q^\star) = \mathrm{KL}(P\|P_\rho) - \mathrm{KL}(P\|Q^\star) \geq \mathrm{KL}(P\|P_\rho) - \mathrm{KL}(P\|R) = \mathbb{E}_P[\log dR] - \mathbb{E}_P[\log dP_\rho].
$$

Substituting the definition of $dR$ yields:
$$
\Delta_Z(Q^\star) \geq -\log C_R - \frac{1}{\lambda}\mathbb{E}_P[c(X,Y,Z)].
$$

Using the inequality $-\log u \geq \frac{1}{2}(u-1)^2 - (u-1)$ for all $u \leq 1$, and setting $u = C_R$ (since $c \geq 0$ ensures $C_R \leq 1$), we obtain:
$$
\Delta_Z(Q^\star) \geq \frac{1}{2}\left(\mathbb{E}_{P_\rho}[e^{-c(X,Y,Z)/\lambda}] - 1\right)^2 - \left(\mathbb{E}_{P_\rho}[e^{-c(X,Y,Z)/\lambda}] - 1\right) - \frac{1}{\lambda}\mathbb{E}_P[c(X,Y,Z)]
$$
$$
= C_\lambda - \mathbb{E}_{P_\rho}[e^{-c(X,Y,Z)/\lambda}] + 1 - \frac{1}{\lambda}\mathbb{E}_P[c(X,Y,Z)].
$$

Finally, applying the inequality $e^{-v} \leq 1 - v + \frac{1}{2}v^2$ for $v = c(X,Y,Z)/\lambda \geq 0$ gives:
$$
\Delta_Z(Q^\star) \geq C_\lambda - \mathbb{E}_{P_\rho}\left[1 - \frac{c(X,Y,Z)}{\lambda} + \frac{c(X,Y,Z)^2}{2\lambda^2}\right] + 1 - \frac{1}{\lambda}\mathbb{E}_P[c(X,Y,Z)]
$$
$$
= C_\lambda + \frac{1}{\lambda}\left(\mathbb{E}_{P_\rho}[c(X,Y,Z)] - \mathbb{E}_P[c(X,Y,Z)]\right) - \frac{1}{2\lambda^2}\mathbb{E}_{P_\rho}[c(X,Y,Z)^2] = C_\lambda + \frac{\delta_c}{\lambda} - \frac{C}{\lambda^2}.
$$
This completes the proof.
\end{proof}

\subsection{Theorem~\ref{thm:gaussian_recovery}}

\begin{proof}
By the equivalence of the global and local bridge formulations, $Q^\star = P$ if and only if the local bridge $\pi_x^\star$ recovers the true conditional law $P_{Y,Z\mid X=x}$ for all $x$. Because the conditional marginals $P_{Y\mid X=x}$ and $P_{Z\mid X=x}$ are fixed by the marginal constraints $P_{X,Y}$ and $P_{X,Z}$, any admissible Gaussian coupling $\pi_x$ is entirely determined by its conditional covariance $q = \mathrm{Cov}_{\pi_x}(Y,Z \mid X=x)$.

Let $\sigma_{Y\mid X}^2$ and $\sigma_{Z\mid X}^2$ denote the conditional variance of $Y$ and $Z$ given $X$. The KL divergence between the coupled Gaussian $\pi_x$ and the independent baseline $\rho_x$ is:
$$
 \mathrm{KL}(\pi_x \| \rho_x) = -\frac{1}{2}\log\Big(1 - \frac{q^2}{\sigma_{Y\mid X}^2 \sigma_{Z\mid X}^2}\Big).
$$

The expected cost under $\pi_x$ takes the form $\mathbb{E}_{\pi_x}[c(Y,Z)] = K - \beta q$, where $K$ depends only on the fixed marginal variances. The cross-term coefficient $\beta$ depends on the chosen cost:\\
$\bullet$ For $c = c_{\mathrm{oracle}}$, we trivially have $\beta = \mathrm{sgn}(\sigma_{YZ\mid X})$.\\
$\bullet$ For $c = c_{\mathrm{align}}$, the conditional expectations are $\mathbb{E}[X\mid Y=y] = \gamma_Y y$ and $\mathbb{E}[X\mid Z=z] = \gamma_Z z$, where $\gamma_Y = \Sigma_{XY}\Sigma_{YY}^{-1}$ and $\gamma_Z = \Sigma_{XZ}\Sigma_{ZZ}^{-1}$. The cross-term of the squared $\ell_2$ norm yields $\beta = 2\gamma_Y^\top \gamma_Z$. Because $\Sigma_{YY}$ and $\Sigma_{ZZ}$ are positive scalars, we have $\mathrm{sgn}(\beta) = \mathrm{sgn}(\Sigma_{YX}\Sigma_{XZ})$. 

The local bridge problem \eqref{eq:local_bridge} thus minimizes the objective $J(q) = K - \beta q + \lambda \mathrm{KL}(\pi_x \| \rho_x)$. Hence
$$
\frac{\partial J}{\partial q} = -\beta + \frac{\lambda q}{\sigma_{Y\mid X}^2 \sigma_{Z\mid X}^2 - q^2} = 0.
$$

For exact recovery, we require the optimal covariance to match the true conditional covariance, $q = \sigma_{YZ\mid X}$. Substituting this into the first-order condition and noting that the denominator becomes the determinant of the true conditional covariance matrix, $|\Sigma_{YZ\mid X}|$, solving for $\lambda$ yields
$$
\lambda^\star = \frac{\beta |\Sigma_{YZ\mid X}|}{\sigma_{YZ\mid X}}.
$$
We see that the exact recovery is possible if and only if $\lambda^\star > 0$ if and only if $\mathrm{sgn}(\beta) = \mathrm{sgn}(\sigma_{YZ\mid X})$.\\ 
$\bullet$  If $c = c_{\mathrm{oracle}}$, $\lambda^\star = |\Sigma_{YZ\mid X}| / |\sigma_{YZ\mid X}| > 0$ holds universally.\\
$\bullet$  If $c = c_{\mathrm{align}}$, the condition $\sigma_{YZ\mid X}\,\Sigma_{YX}\Sigma_{XZ} > 0$ guarantees $\mathrm{sgn}(\beta) = \mathrm{sgn}(\sigma_{YZ\mid X})$.

Under either condition, setting $\lambda = \lambda^\star$ uniquely minimizes the local bridge objective at the true joint law for all $x$, resulting in $Q^\star = P$.
\end{proof}

\section{Experiment studies}
\label{supp:simul}

This section consists of the following categories: 
\begin{itemize}
    \item Section~\ref{app:implementation_details} explains the compatibility cost, data preprocessing procedure, and neural network structure in detail for each experiment.
    \item Section~\ref{app:discrete_protocol} explains the binning procedure used for continuous variables.   
    \item Section~\ref{app:metrics} specifically defines the evaluation metrics.
    \item Section~\ref{app:lambda_selection} introduces the procedure of selecting $\lambda$.
    \item Section~\ref{app:compute_resources} reports the computational resources used for the experiments.
\end{itemize}

\subsection{Architectures, costs, and preprocessing}
\label{app:implementation_details}

This section specifies the experimental details. Recall the recipient file is denoted by $D^A=\{(X_i^A,Y_i^A)\}_{i=1}^{n_A}$, 
the donor file by $D^B=\{(X_j^B,Z_j^B)\}_{j=1}^{n_B}$, 
and the fully observed test file by $D^T=\{(X_\ell^T,Y_\ell^T,Z_\ell^T)\}_{\ell=1}^{n_T}$. 
Equivalently, we write
\[
    D^A=(X^A,Y^A),\qquad
    D^B=(X^B,Z^B),\qquad
    D^T=(X^T,Y^T,Z^T).
\]
The test file \(D^T\) is used only for final evaluation and is not used for bridge learning, posterior completion, downstream training, or selecting \(\lambda\). For \textsf{SB (Strong)}, an additional fully observed auxiliary set of size \(|D^A|=|D^B|=10\)k is used for \(\lambda\)-selection only (see Appendix~\ref{app:lambda_selection}).

The dimensions of \(X,Y,Z\) are denoted by \(d_X,d_Y,d_Z\). For a continuous variable \(W\), \(\bar W\) and \(s_W\) denote the empirical mean and empirical standard deviation computed from the file in which \(W\) is observed. The standardized value is denoted by
\[
    W^{\mathrm{std}}=(W-\bar W)/s_W .
\]
Standardization constants are computed from the corresponding observed training file and then applied to validation and test data. Categorical variables are one-hot encoded. A \(K\)-state discrete variable is represented either by an index in \(\{1,\ldots,K\}\) or by the corresponding \(K\)-dimensional one-hot vector.

All neural models are implemented in PyTorch. We parameterize the potential \(A_p(x,y)\) and the potential \(B(x,z)\) as positive multilayer perceptrons with two hidden layers of width 64 and Softplus output activation; the output of \(B(x,z)\) is normalized for identifiability. For continuous Gaussian and synthetic experiments, the input dimensions are \(d_X+d_Y\) for \(A_p\) and \(d_X+d_Z\) for \(B\). For real-data experiments, the same hidden-layer design is used after replacing the input dimension by the corresponding latent, one-hot, or discrete auxiliary representation.

The downstream predictor is denoted by \(f_\psi\). It is an MLP with hidden widths 32 and 16 and ReLU activations. Regression experiments use mean squared error. Binary classification experiments use binary cross-entropy. \textsf{NAIVE} trains \(f_\psi\) with input \(X\). \textsf{SB} and \textsf{ORACLE} train \(f_\psi\) with input \((X,Z)\). The same downstream architecture is used for \textsf{NAIVE}, \textsf{SB}, and \textsf{ORACLE}.

\paragraph{Gaussian experiment.}
\((X,Y,Z)\) is sampled from a zero-mean trivariate Gaussian distribution with unit variances,
\[
    \sigma_{XY}=\sigma_{XZ}=0.5,
    \qquad
    \sigma_{YZ}\in\{0.3,0.5,0.7,0.9\}.
\]
Each setting uses 30,000 samples, split equally into \(D^A\), \(D^B\), and \(D^T\) (10,000 each). The variables \(Y^A\) and \(Z^B\) are standardized using their observed training-file values before bridge learning, and the imputed values are transformed back to the original scale for downstream training and evaluation.

For the Gaussian bridge, we use the empirical alignment cost in standardized coordinates while preserving the raw-scale covariance structure. Specifically, if \(y^{\mathrm{std}}\) and \(z^{\mathrm{std}}\) denote the standardized versions of \(y\) and \(z\), and \(s_Y,s_Z,s_{XY},s_{XZ}\) denote the corresponding empirical standard deviations and covariances, the cost is implemented as
\[
    c(x,y,z)
    =
    \left(
    s_{XY}s_Y y^{\mathrm{std}}
    -
    s_{XZ}s_Z z^{\mathrm{std}}
    \right)^2.
\]
This is the sample-level implementation of the Gaussian alignment cost
\[
    c_{\mathrm{align}}(y,z)
    =
    \left\|
    \mathbb E[X\mid Y=y]-\mathbb E[X\mid Z=z]
    \right\|_2^2
\]
under the trivariate Gaussian model.

\paragraph{Synthetic regression experiments.}
\(X=(X_1,X_2)\sim N(0,I_2)\), and the noise pair \((e_1,e_2)\) has common variance \(\sigma^2_{\mathrm{noise}}\in\{0.1,0.5\}\) and correlation \(\rho_{\mathrm{noise}}\in\{0.3,0.6,0.9\}\). The three data-generating mechanisms are
\[
\begin{array}{ll}
\mathrm{SYN1}:&
Y=X_1+2X_2+e_1,
\qquad
Z=2X_1+X_2+e_2,\\[2mm]
\mathrm{SYN2}:&
Y=\sin(X_1)+X_2^2+e_1,
\qquad
Z=\sin(X_1+X_2)+e_2,\\[2mm]
\mathrm{SYN3}:&
Y=\log(1+e^{X_1})+X_2+e_1,
\qquad
Z=\frac12\log(1+e^{X_1})+X_2^2+e_2.
\end{array}
\]
Each synthetic setting uses 30,000 samples, split equally into \(D^A\), \(D^B\), and \(D^T\).

The cost compares standardized residuals after subtracting the estimated conditional means explained by \(X\). With
\[
    u_Y(y)=\frac{y-\bar Y^A}{s_Y^A},
    \qquad
    u_Z(z)=\frac{z-\bar Z^B}{s_Z^B},
\]
let \(\hat m_Y(x)\) be the fitted conditional mean of \(u_Y(Y)\) given \(X=x\), trained on \(D^A\), and let \(\hat m_Z(x)\) be the fitted conditional mean of \(u_Z(Z)\) given \(X=x\), trained on \(D^B\). The residual cost is
\[
    c(x,y,z)
    =
    \frac{
    \left[
    \{u_Y(y)-\hat m_Y(x)\}
    -
    \{u_Z(z)-\hat m_Z(x)\}
    \right]^2
    }{K_{\mathrm{res}}},
\]
where \(K_{\mathrm{res}}\) is the empirical variance of the numerator over the Monte Carlo triples used for bridge training.

\paragraph{CelebA.}
\(X\) is an 8-dimensional VAE latent embedding, \(Y\) is the binary \texttt{Attractive} attribute, and \(Z\) is the 32-state discrete auxiliary variable induced by the five binary attributes \texttt{Young}, \texttt{High Cheekbones}, \texttt{Arched Eyebrows}, \texttt{Heavy Makeup}, and \texttt{No Beard}. Equivalently, each \(Z\)-state corresponds to one binary vector in \(\{0,1\}^5\), so there are \(2^5=32\) possible auxiliary states. The dataset is split into 40\% for \(D^A\), 40\% for \(D^B\), and 20\% for \(D^T\).

\paragraph{Adult.}
\(Y\) indicates income above \$50K. The common covariates \(X\) consist of age and the categorical variables sex, race, native-country, marital-status, and relationship. The auxiliary variable \(Z\) is built from education-num, hours-per-week, capital-gain, capital-loss, education, occupation, and workclass. Numerical variables are standardized, categorical variables are one-hot encoded, and the resulting auxiliary representation is discretized into \(K=32\) clusters by \(K\)-means. The split uses 16,000 samples for \(D^A\), 16,000 samples for \(D^B\), and the remainder for \(D^T\).

\paragraph{Training hyperparameters.}
Unless otherwise stated, the bridge potentials are trained using Adam with learning rate \(5\times 10^{-4}\). The bridge mini-batch size is 256, and the Monte Carlo sample size for approximating the inner expectations is 256. We use a scale regularization coefficient \(10^{-3}\), loss weights \(w_A=2\) and \(w_B=1\), and gradient clipping with maximum norm 5. For posterior completion, each missing value is sampled from a candidate pool of size 1024, i.e., \(M_Z=1024\) and \(M_Y=1024\), unless the corresponding support is finite and directly enumerated. Imputation is performed in batches of size 2000. For downstream prediction, the downstream MLP is trained using Adam with learning rate \(10^{-3}\), mini-batch size 256, and at most 500 epochs with early stopping patience 30 on a 20\% validation split of the downstream training data.

\paragraph{Completion.}
Each missing \(Z_i^A\) is sampled from a candidate pool $\mathcal C_Z=\{z_1,\ldots,z_{M_Z}\}$
For an observed pair \((x_i^A,y_i^A)\), the unnormalized posterior weight of candidate \(z_k\) is $\omega_{ik}^A
    =
    \exp\left\{
    \log B(x_i^A,z_k)
    -
    \frac{c(x_i^A,y_i^A,z_k)}{\lambda}
    \right\}$, 
and the normalized probability is $p_{ik}^A
    =
    \frac{\omega_{ik}^A}
    {\sum_{\ell=1}^{M_Z}\omega_{i\ell}^A}$. 
The completed value \(\hat Z_i^A\) is sampled from \(\mathcal C_Z\) according to \(\{p_{ik}^A\}_{k=1}^{M_Z}\). For donor-side completion, each missing \(Y_j^B\) is sampled from a candidate pool $\mathcal C_Y=\{y_1,\ldots,y_{M_Y}\}$.  For an observed pair \((x_j^B,z_j^B)\), the unnormalized posterior weight of candidate \(y_k\) is
\[
    \omega_{jk}^B
    =
    \exp\left\{
    \log A_p(x_j^B,y_k)
    +
    \log \hat d(x_j^B,y_k,z_j^B)
    -
    \frac{c(x_j^B,y_k,z_j^B)}{\lambda}
    \right\}.
\]
where $\hat d(x,y,z)=\frac{\hat p(y\mid x)}{\hat p(z\mid x)}$. Note that $\hat p(z|x)$ is a constant.

\paragraph{Density-ratio estimation.}
The conditional densities \(\hat p_{Y\mid X}\) and \(\hat p_{Z\mid X}\) are estimated separately from \(D^A\) and \(D^B\), respectively. For the Gaussian experiment, we assume Gaussian conditionals and plug in their closed-form maximum likelihood estimators. For the synthetic regression experiments, we again assume Gaussian conditionals: the conditional means are estimated by neural networks, and the conditional standard deviations are taken as the corresponding residual MSE. For Adult and CelebA, where \(Y\) and \(Z\) are categorical, we fit neural classifiers in each marginal and use the output softmax probabilities directly.

\paragraph{Pooled completed data.}
The completed recipient and donor files are
\[
    \hat D^A
    =
    \{(X_i^A,Y_i^A,\hat Z_i^A)\}_{i=1}^{n_A},
    \qquad
    \hat D^B
    =
    \{(X_j^B,\hat Y_j^B,Z_j^B)\}_{j=1}^{n_B}.
\]
The pooled completed data are denoted by
\[
    \hat D^T=(X',Y',Z')
    :=
    (X^A\cup X^B,\;Y^A\cup \hat Y^B,\;\hat Z^A\cup Z^B),
\]
on which the downstream predictor \(f_\psi\) is trained. For \(\lambda\), we use the selection procedure described in Appendix~\ref{app:lambda_selection}.

To see further details, refer to the code scripts in Supplementary.

\subsection{Common discrete evaluation protocol}
\label{app:discrete_protocol}

Because \textsf{GO} is defined on a discrete support, Gaussian and synthetic benchmark comparisons are evaluated in a common finite state space. We use the following bin configurations:
\[
    (n_X,n_Y,n_Z)
    \in
    \{(3,3,3),(4,3,4),(5,3,5),(6,3,4),(4,4,4),(4,3,6),(5,4,5),(6,4,6)\}.
\]
For a given configuration, \(n_X\) is the number of bins used for each coordinate of \(X\), \(n_Y\) is the number of bins used for \(Y\), and \(n_Z\) is the number of bins used for \(Z\). If \(X\) is multivariate, each coordinate is binned separately and the resulting vector of bin labels is used as the discrete \(X\)-state.

To avoid leakage, bin boundaries are computed without using \(D^T\). Boundaries for \(X\) are computed from the pooled observed \(X\)-values in \(D^A\cup D^B\). Boundaries for \(Y\) are computed only from the observed \(Y^A\) values in \(D^A\). Boundaries for \(Z\) are computed only from the observed \(Z^B\) values in \(D^B\). The same bin maps are applied to completed values and to evaluation-only true values.

For \textsf{SB} recipient-side reconstruction, posterior probabilities are aggregated at the bin level. Let \(g_Z\) be the bin map for \(Z\). For a recipient-side sample \(i\), the posterior probability assigned to the \(k\)-th \(Z\)-bin is
\[
    \hat p_i^A(k)
    =
    \sum_{\ell:\,g_Z(z_\ell)=k}
    p_{i\ell}^A .
\]
The predicted \(Z\)-bin is
\[
    \hat z^A_i
    =
    \arg\max_{k\in\{1,\ldots,n_Z\}}
    \hat p_i^A(k).
\]
Downstream comparisons use the same binned \(Z\)-representation for every method.

The reported Gaussian and synthetic benchmark tables average over the eight bin configurations and also report the best configuration when applicable. For real-data experiments, no additional binning is applied. CelebA uses the intrinsic 32-state auxiliary variable, where each state is one element of \(\{0,1\}^5\). Thus, a reconstructed \(Z\)-state is counted as correct only when the entire five-dimensional binary attribute vector is recovered correctly. Adult uses the 32-state auxiliary variable obtained by \(K\)-means, and a reconstructed \(Z\)-state is counted as correct when the predicted cluster index matches the true cluster index. Unless otherwise stated, the uncertainty values reported in the main tables are computed across independent random seeds after fixing the evaluation protocol, and do not represent variability across
bin configurations.

\subsection{Evaluation metrics}
\label{app:metrics}

This section defines the metrics used in the figures and tables. For continuous downstream regression in the Gaussian and synthetic experiments, we report the coefficient of determination \(R^2\) and the root mean squared error. Given held-out test samples \(\{(X_i^T,Y_i^T,Z_i^T)\}_{i=1}^{n_T}\) and downstream predictions \(\hat Y_i^T\), the coefficient of determination is
\[
    R^2
    =
    1-
    \frac{
    \sum_{i=1}^{n_T}(Y_i^T-\hat Y_i^T)^2
    }{
    \sum_{i=1}^{n_T}(Y_i^T-\bar Y^T)^2
    },
    \qquad
    \bar Y^T=\frac{1}{n_T}\sum_{i=1}^{n_T}Y_i^T.
\]
The downstream root mean squared error is
\[
    \mathrm{RMSE}_{\mathrm{down}}
    =
    \left\{
    \frac{1}{n_T}
    \sum_{i=1}^{n_T}
    (Y_i^T-\hat Y_i^T)^2
    \right\}^{1/2}.
\]
In Figure~\ref{fig:gaussian_experiments}, the heatmap reports the downstream uplift
\[
    R^2_{\mathrm{SB}}-R^2_{\mathrm{NAIVE}},
\]
and the overlaid curve reports \(\mathrm{RMSE}_{\mathrm{down}}\).

For donor-side continuous imputation, we report the RMSE between the true and imputed \(Y^B\) values:
\[
    \mathrm{RMSE}^{B}
    =
    \left\{
    \frac{1}{n_B}
    \sum_{j=1}^{n_B}
    (Y_{j,\mathrm{true}}^B-\hat Y_j^B)^2
    \right\}^{1/2}.
\]
Here \(Y_{j,\mathrm{true}}^B\) is observed only for evaluation in the fully generated or held-out experimental data, not by the imputation algorithm.

For binary real-data experiments, we report ROC-AUC. The downstream AUC, denoted Down. AUC, is computed from the true test labels \(Y_i^T\in\{0,1\}\) and the predicted scores \(\hat s_i^T\). Equivalently,
\[
    \mathrm{AUC}
    =
    \Pr(\hat s^+>\hat s^-)
    +
    \frac{1}{2}\Pr(\hat s^+=\hat s^-),
\]
where \(\hat s^+\) and \(\hat s^-\) denote the scores of randomly chosen positive and negative test samples, respectively. For donor-side binary imputation, \(\mathrm{AUC}^B\) is computed analogously using the true donor-side labels \(Y_{j,\mathrm{true}}^B\) and the imputed probability scores
\[
    \hat s_j^B
    =
    \Pr(\hat Y_j^B=1\mid X_j^B,Z_j^B).
\]

When we report the fraction of the \textsf{NAIVE}-to-\textsf{ORACLE} gap recovered by \textsf{SB}, we use
\[
    \mathrm{GapClosure}
    =
    \frac{
    M_{\textsf{SB}}-M_{\textsf{NAIVE}}
    }{
    M_{\textsf{ORACLE}}-M_{\textsf{NAIVE}}
    },
\]
where \(M=R^2\) for continuous regression experiments and \(M=\mathrm{AUC}\) for binary classification experiments. This quantity is interpreted only when \(M_{\textsf{ORACLE}}>M_{\textsf{NAIVE}}\).

All reported values are averaged over 10 independent runs with different random seeds, and the uncertainty values denote one sample standard deviation across the runs. Specifically, for metric values \(m_1,\ldots,m_{10}\), we report
\[
    \bar m \pm s_m,
    \qquad
    \bar m=\frac{1}{10}\sum_{r=1}^{10}m_r,
    \qquad
    s_m=
    \left\{
    \frac{1}{9}\sum_{r=1}^{10}(m_r-\bar m)^2
    \right\}^{1/2}.
\]

\subsection{Selecting \texorpdfstring{$\lambda$}{lambda} via auxiliary triplets}
\label{app:lambda_selection}

This section details the validation-based tuning procedure referenced in the main text (Section~\ref{sec:simul}, ``Choice of \(\lambda\)''). We assume access to an auxiliary set of fully observed triplets $C = \{(X_i, Y_i, Z_i)\}_{i=1}^{n_C} \sim P_{X,Y,Z}$, 
disjoint from \(D^A\), \(D^B\), and the held-out test file \(D^T\). The candidate set is $\Lambda
    =
    \{10^{-2},10^{-2+1/3},10^{-2+2/3},\ldots,10^1\}$, 
namely 10 log-spaced values between \(0.01\) and \(10\).

For each \(\lambda \in \Lambda\), the bridge potentials are trained on \((D^A, D^B)\), and the learned posterior is used to construct the pooled imputed dataset $\hat D^T(\lambda) = \hat D^A(\lambda) \cup \hat D^B(\lambda)$.
A downstream predictor \(f_\psi^\lambda\) is trained on \(\hat D^T(\lambda)\) to predict \(Y'\) from \((X', Z')\), and evaluated on \(C\) using \(R^2\) for regression and AUC for binary classification:
\[
    \mathrm{Score}_C(\lambda)
    =
    \begin{cases}
        R^2\!\left(Y_C, f_\psi^\lambda(X_C, Z_C)\right) & \text{(regression)},\\[2pt]
        \mathrm{AUC}\!\left(Y_C, f_\psi^\lambda(X_C, Z_C)\right) & \text{(binary classification)}.
    \end{cases}
\]
The selected value is $\hat\lambda=
    \arg\max_{\lambda\in\Lambda}
    \mathrm{Score}_C(\lambda)$. 
The same selection protocol is applied independently within each random seed, and \(D^T\) is reserved for final evaluation only.

\paragraph{\textsf{SB (Weak)} vs.\ \textsf{SB (Strong)}.}
The two variants differ only in the size of \(C\):
\textsf{SB (Weak)} uses a small set with \(n_C = 50\), reflecting the typical practical regime in statistical matching where only a few fully observed triplets are publicly available;
\textsf{SB (Strong)} uses a large set with \(n_C = |D^A| = |D^B| = 10\)k, serving as an idealized reference. Comparing the two isolates the sensitivity of \(\lambda\)-selection to the size of the auxiliary set.

\subsection{Computational resources}
\label{app:compute_resources}

All experiments were implemented in PyTorch and run on a GPU-equipped workstation. 
The workstation had four NVIDIA L4 GPUs, each with approximately 23GB of memory, an Intel Xeon Gold 6526Y CPU with 32 logical CPUs, and 256GB of system memory. 
The NVIDIA driver version was 560.35.03 and the CUDA version was 12.6. 
Data preprocessing and metric aggregation were performed on CPUs, while bridge training, posterior completion, and downstream predictor training were performed on GPUs. 
A single run, consisting of bridge training, posterior completion, downstream predictor training, and final evaluation for one dataset configuration, random seed, and value of \(\lambda\), took approximately 5 minutes.

\section{Additional Experimental Results}
\label{supp:simul-tables}

This section provides the full MM1--MM5 benchmark results omitted from the main text due to space constraints. 
The results cover the Gaussian, synthetic, and real-data experiments.

\begin{table}[htbp]
\centering
\caption{Gaussian benchmark with all mixed-method variants. Down. $R^2$ is the coefficient of determination for predicting $Y$; $RMSE^B$ is the donor-side imputation error for $\hat Y^B$. Boldface indicates the best-performing mixed-method variant for each metric.}
\label{tab:gaussian_mm_all}
\normalsize
\setlength{\tabcolsep}{5.5pt}
\renewcommand{\arraystretch}{1.10}
\resizebox{\textwidth}{!}{%
\begin{tabular}{ccccccc}
\toprule
$\sigma_{YZ}$
& Metric
& \textsf{MM1}
& \textsf{MM2}
& \textsf{MM3}
& \textsf{MM4}
& \textsf{MM5} \\
\midrule

\multirow{2}{*}{\textbf{0.9}}
& Down. $R^2$
& $0.248 \pm 0.010$
& $0.253 \pm 0.010$
& $0.262 \pm 0.010$
& $0.249 \pm 0.008$
& {\bfseries\boldmath $0.264 \pm 0.009$} \\
& $RMSE^B$
& $1.228 \pm 0.001$
& {\bfseries\boldmath $0.866 \pm 0.008$}
& $1.224 \pm 0.003$
& $1.231 \pm 0.004$
& $1.223 \pm 0.002$ \\

\midrule

\multirow{2}{*}{\textbf{0.7}}
& Down. $R^2$
& $0.258 \pm 0.015$
& $0.255 \pm 0.011$
& {\bfseries\boldmath $0.261 \pm 0.008$}
& $0.253 \pm 0.008$
& $0.258 \pm 0.008$ \\
& $RMSE^B$
& $1.224 \pm 0.006$
& {\bfseries\boldmath $0.866 \pm 0.009$}
& $1.224 \pm 0.003$
& $1.232 \pm 0.005$
& $1.223 \pm 0.002$ \\

\midrule

\multirow{2}{*}{\textbf{0.5}}
& Down. $R^2$
& $0.257 \pm 0.006$
& {\bfseries\boldmath $0.260 \pm 0.007$}
& $0.260 \pm 0.004$
& $0.256 \pm 0.003$
& $0.260 \pm 0.006$ \\
& $RMSE^B$
& $1.220 \pm 0.010$
& {\bfseries\boldmath $0.868 \pm 0.010$}
& $1.226 \pm 0.004$
& $1.233 \pm 0.006$
& $1.224 \pm 0.005$ \\

\midrule

\multirow{2}{*}{\textbf{0.3}}
& Down. $R^2$
& $0.259 \pm 0.005$
& {\bfseries\boldmath $0.260 \pm 0.006$}
& $0.258 \pm 0.005$
& $0.259 \pm 0.005$
& $0.259 \pm 0.006$ \\
& $RMSE^B$
& $1.221 \pm 0.006$
& {\bfseries\boldmath $0.867 \pm 0.009$}
& $1.225 \pm 0.003$
& $1.232 \pm 0.005$
& $1.222 \pm 0.004$ \\

\bottomrule
\end{tabular}
}
\end{table}

\begin{table}[htbp]
\centering
\caption{SYN1 benchmark with all mixed-method variants. Down. $R^2$ is the downstream test $R^2$ for predicting $Y$; $RMSE^B$ is the donor-side imputation error for $\hat Y^B$. Boldface indicates the best-performing mixed-method variant for each metric; in cases of tied means, the variant with the smallest standard deviation is selected.}
\label{tab:syn1_mm_all}
\normalsize
\setlength{\tabcolsep}{4.0pt}
\renewcommand{\arraystretch}{1.10}
\resizebox{\textwidth}{!}{%
\begin{tabular}{cccccccc}
\toprule
$\sigma^2_{\text{noise}}$
& $\rho_{\text{noise}}$
& Metric
& \textsf{MM1}
& \textsf{MM2}
& \textsf{MM3}
& \textsf{MM4}
& \textsf{MM5} \\
\midrule
\multirow{2}{*}{\textbf{0.5}}
& \multirow{2}{*}{\textbf{0.9}}
& Down. $R^2$
& $0.907 \pm 0.002$
& {\bfseries\boldmath $0.908 \pm 0.001$}
& $0.907 \pm 0.002$
& $0.907 \pm 0.001$
& $0.907 \pm 0.002$ \\
&
& $RMSE^B$
& $1.005 \pm 0.006$
& {\bfseries\boldmath $0.706 \pm 0.007$}
& $1.001 \pm 0.006$
& $1.008 \pm 0.009$
& $1.003 \pm 0.005$ \\
\midrule
\multirow{2}{*}{\textbf{0.5}}
& \multirow{2}{*}{\textbf{0.6}}
& Down. $R^2$
& $0.906 \pm 0.001$
& $0.907 \pm 0.001$
& $0.907 \pm 0.001$
& {\bfseries\boldmath $0.907 \pm 0.001$}
& $0.907 \pm 0.001$ \\
&
& $RMSE^B$
& $1.005 \pm 0.009$
& {\bfseries\boldmath $0.705 \pm 0.006$}
& $0.999 \pm 0.007$
& $1.007 \pm 0.011$
& $1.000 \pm 0.004$ \\
\midrule
\multirow{2}{*}{\textbf{0.5}}
& \multirow{2}{*}{\textbf{0.3}}
& Down. $R^2$
& $0.906 \pm 0.001$
& {\bfseries\boldmath $0.907 \pm 0.001$}
& $0.907 \pm 0.001$
& $0.907 \pm 0.001$
& $0.907 \pm 0.001$ \\
&
& $RMSE^B$
& $1.003 \pm 0.008$
& {\bfseries\boldmath $0.704 \pm 0.005$}
& $0.998 \pm 0.008$
& $1.006 \pm 0.012$
& $0.999 \pm 0.005$ \\
\midrule
\multirow{2}{*}{\textbf{0.1}}
& \multirow{2}{*}{\textbf{0.9}}
& Down. $R^2$
& $0.980 \pm 0.000$
& {\bfseries\boldmath $0.980 \pm 0.000$}
& $0.980 \pm 0.000$
& $0.980 \pm 0.000$
& $0.980 \pm 0.000$ \\
&
& $RMSE^B$
& $0.450 \pm 0.003$
& {\bfseries\boldmath $0.316 \pm 0.003$}
& $0.448 \pm 0.003$
& $0.454 \pm 0.005$
& $0.450 \pm 0.002$ \\
\midrule
\multirow{2}{*}{\textbf{0.1}}
& \multirow{2}{*}{\textbf{0.6}}
& Down. $R^2$
& $0.980 \pm 0.000$
& {\bfseries\boldmath $0.980 \pm 0.000$}
& $0.980 \pm 0.000$
& $0.980 \pm 0.000$
& $0.980 \pm 0.000$ \\
&
& $RMSE^B$
& $0.449 \pm 0.003$
& {\bfseries\boldmath $0.315 \pm 0.003$}
& $0.447 \pm 0.003$
& $0.454 \pm 0.006$
& $0.449 \pm 0.001$ \\
\midrule
\multirow{2}{*}{\textbf{0.1}}
& \multirow{2}{*}{\textbf{0.3}}
& Down. $R^2$
& $0.980 \pm 0.000$
& {\bfseries\boldmath $0.980 \pm 0.000$}
& $0.980 \pm 0.000$
& $0.980 \pm 0.000$
& $0.980 \pm 0.000$ \\
&
& $RMSE^B$
& $0.447 \pm 0.003$
& {\bfseries\boldmath $0.315 \pm 0.002$}
& $0.446 \pm 0.003$
& $0.454 \pm 0.007$
& $0.449 \pm 0.002$ \\
\bottomrule
\end{tabular}
}
\end{table}

\begin{table}[htbp]
\centering
\caption{SYN2 benchmark with all mixed-method variants. Down. $R^2$ is the downstream test $R^2$ for predicting $Y$; $RMSE^B$ is the donor-side imputation error for $\hat Y^B$. Boldface indicates the best-performing mixed-method variant for each metric.}
\label{tab:syn2_mm_all}
\normalsize
\setlength{\tabcolsep}{4.0pt}
\renewcommand{\arraystretch}{1.10}
\resizebox{\textwidth}{!}{%
\begin{tabular}{cccccccc}
\toprule
$\sigma^2_{\epsilon}$
& $\rho$
& Metric
& \textsf{MM1}
& \textsf{MM2}
& \textsf{MM3}
& \textsf{MM4}
& \textsf{MM5} \\
\midrule

\multirow{2}{*}{\textbf{0.5}}
& \multirow{2}{*}{\textbf{0.9}}
& Down. $R^2$
& {\bfseries\boldmath $0.657 \pm 0.013$}
& $0.354 \pm 0.026$
& $0.592 \pm 0.025$
& $0.593 \pm 0.010$
& $0.559 \pm 0.010$ \\
&
& $RMSE^B$
& $2.223 \pm 0.009$
& {\bfseries\boldmath $1.588 \pm 0.025$}
& $2.244 \pm 0.013$
& $2.271 \pm 0.017$
& $2.257 \pm 0.028$ \\

\midrule

\multirow{2}{*}{\textbf{0.5}}
& \multirow{2}{*}{\textbf{0.6}}
& Down. $R^2$
& {\bfseries\boldmath $0.641 \pm 0.008$}
& $0.348 \pm 0.032$
& $0.591 \pm 0.010$
& $0.588 \pm 0.021$
& $0.571 \pm 0.030$ \\
&
& $RMSE^B$
& $2.258 \pm 0.009$
& {\bfseries\boldmath $1.585 \pm 0.024$}
& $2.243 \pm 0.014$
& $2.271 \pm 0.014$
& $2.253 \pm 0.027$ \\

\midrule

\multirow{2}{*}{\textbf{0.5}}
& \multirow{2}{*}{\textbf{0.3}}
& Down. $R^2$
& {\bfseries\boldmath $0.647 \pm 0.010$}
& $0.344 \pm 0.027$
& $0.580 \pm 0.012$
& $0.589 \pm 0.015$
& $0.559 \pm 0.024$ \\
&
& $RMSE^B$
& $2.223 \pm 0.040$
& {\bfseries\boldmath $1.584 \pm 0.023$}
& $2.242 \pm 0.014$
& $2.270 \pm 0.013$
& $2.252 \pm 0.026$ \\

\midrule

\multirow{2}{*}{\textbf{0.1}}
& \multirow{2}{*}{\textbf{0.9}}
& Down. $R^2$
& {\bfseries\boldmath $0.759 \pm 0.008$}
& $0.405 \pm 0.046$
& $0.624 \pm 0.024$
& $0.613 \pm 0.015$
& $0.603 \pm 0.028$ \\
&
& $RMSE^B$
& $2.044 \pm 0.018$
& {\bfseries\boldmath $1.457 \pm 0.024$}
& $2.037 \pm 0.017$
& $2.092 \pm 0.021$
& $2.043 \pm 0.031$ \\

\midrule

\multirow{2}{*}{\textbf{0.1}}
& \multirow{2}{*}{\textbf{0.6}}
& Down. $R^2$
& {\bfseries\boldmath $0.757 \pm 0.008$}
& $0.392 \pm 0.059$
& $0.612 \pm 0.020$
& $0.615 \pm 0.010$
& $0.602 \pm 0.027$ \\
&
& $RMSE^B$
& $2.048 \pm 0.057$
& {\bfseries\boldmath $1.456 \pm 0.024$}
& $2.038 \pm 0.018$
& $2.092 \pm 0.020$
& $2.043 \pm 0.032$ \\

\midrule

\multirow{2}{*}{\textbf{0.1}}
& \multirow{2}{*}{\textbf{0.3}}
& Down. $R^2$
& {\bfseries\boldmath $0.756 \pm 0.006$}
& $0.404 \pm 0.053$
& $0.624 \pm 0.013$
& $0.607 \pm 0.015$
& $0.600 \pm 0.016$ \\
&
& $RMSE^B$
& $2.045 \pm 0.055$
& {\bfseries\boldmath $1.455 \pm 0.023$}
& $2.037 \pm 0.017$
& $2.092 \pm 0.019$
& $2.042 \pm 0.032$ \\

\bottomrule
\end{tabular}
}
\end{table}

\begin{table}[htbp]
\centering
\caption{SYN3 benchmark with all mixed-method variants. Down. $R^2$ is the downstream test $R^2$ for predicting $Y$; $RMSE^B$ is the donor-side imputation error for $\hat Y^B$. Boldface indicates the best-performing mixed-method variant for each metric.}
\label{tab:syn3_mm_all}
\normalsize
\setlength{\tabcolsep}{4.0pt}
\renewcommand{\arraystretch}{1.10}
\resizebox{\textwidth}{!}{%
\begin{tabular}{cccccccc}
\toprule
$\sigma^2_{\epsilon}$
& $\rho$
& Metric
& \textsf{MM1}
& \textsf{MM2}
& \textsf{MM3}
& \textsf{MM4}
& \textsf{MM5} \\
\midrule

\multirow{2}{*}{\textbf{0.5}}
& \multirow{2}{*}{\textbf{0.9}}
& Down. $R^2$
& $0.706 \pm 0.006$
& $0.705 \pm 0.003$
& {\bfseries\boldmath $0.716 \pm 0.010$}
& $0.708 \pm 0.006$
& $0.702 \pm 0.006$ \\
&
& $RMSE^B$
& $1.026 \pm 0.003$
& {\bfseries\boldmath $0.721 \pm 0.007$}
& $1.020 \pm 0.006$
& $1.026 \pm 0.008$
& $1.031 \pm 0.007$ \\

\midrule

\multirow{2}{*}{\textbf{0.5}}
& \multirow{2}{*}{\textbf{0.6}}
& Down. $R^2$
& $0.706 \pm 0.004$
& $0.705 \pm 0.003$
& $0.710 \pm 0.004$
& {\bfseries\boldmath $0.710 \pm 0.003$}
& $0.702 \pm 0.003$ \\
&
& $RMSE^B$
& $1.021 \pm 0.007$
& {\bfseries\boldmath $0.719 \pm 0.006$}
& $1.018 \pm 0.008$
& $1.025 \pm 0.010$
& $1.028 \pm 0.007$ \\

\midrule

\multirow{2}{*}{\textbf{0.5}}
& \multirow{2}{*}{\textbf{0.3}}
& Down. $R^2$
& $0.709 \pm 0.004$
& $0.705 \pm 0.003$
& {\bfseries\boldmath $0.710 \pm 0.003$}
& $0.710 \pm 0.004$
& $0.703 \pm 0.002$ \\
&
& $RMSE^B$
& $1.022 \pm 0.010$
& {\bfseries\boldmath $0.718 \pm 0.005$}
& $1.017 \pm 0.009$
& $1.024 \pm 0.011$
& $1.026 \pm 0.006$ \\

\midrule

\multirow{2}{*}{\textbf{0.1}}
& \multirow{2}{*}{\textbf{0.9}}
& Down. $R^2$
& {\bfseries\boldmath $0.919 \pm 0.004$}
& $0.914 \pm 0.001$
& $0.918 \pm 0.002$
& $0.917 \pm 0.002$
& $0.911 \pm 0.003$ \\
&
& $RMSE^B$
& $0.493 \pm 0.002$
& {\bfseries\boldmath $0.348 \pm 0.004$}
& $0.492 \pm 0.003$
& $0.495 \pm 0.006$
& $0.515 \pm 0.006$ \\

\midrule

\multirow{2}{*}{\textbf{0.1}}
& \multirow{2}{*}{\textbf{0.6}}
& Down. $R^2$
& $0.917 \pm 0.002$
& $0.914 \pm 0.002$
& {\bfseries\boldmath $0.918 \pm 0.004$}
& $0.917 \pm 0.002$
& $0.911 \pm 0.003$ \\
&
& $RMSE^B$
& $0.493 \pm 0.004$
& {\bfseries\boldmath $0.347 \pm 0.004$}
& $0.491 \pm 0.004$
& $0.495 \pm 0.007$
& $0.513 \pm 0.006$ \\

\midrule

\multirow{2}{*}{\textbf{0.1}}
& \multirow{2}{*}{\textbf{0.3}}
& Down. $R^2$
& $0.917 \pm 0.001$
& $0.913 \pm 0.001$
& {\bfseries\boldmath $0.918 \pm 0.004$}
& $0.917 \pm 0.002$
& $0.911 \pm 0.003$ \\
&
& $RMSE^B$
& $0.491 \pm 0.006$
& {\bfseries\boldmath $0.346 \pm 0.004$}
& $0.490 \pm 0.005$
& $0.495 \pm 0.007$
& $0.512 \pm 0.006$ \\

\bottomrule
\end{tabular}
}
\end{table}

\begin{table}[htbp]
\centering
\caption{Real-data benchmarks with all mixed-method variants. Down. AUC is the downstream test AUC for predicting $Y$; $AUC^B$ is the donor-side imputation AUC for $\hat Y^B$. Boldface indicates the best-performing mixed-method variant for each metric.}
\label{tab:real_mm_all}
\normalsize
\setlength{\tabcolsep}{4.5pt}
\renewcommand{\arraystretch}{1.10}
\resizebox{\textwidth}{!}{%
\begin{tabular}{ccccccc}
\toprule
Dataset
& Metric
& \textsf{MM1}
& \textsf{MM2}
& \textsf{MM3}
& \textsf{MM4}
& \textsf{MM5} \\
\midrule

\multirow{2}{*}{\textbf{CelebA}}
& Down. AUC
& $0.765 \pm 0.002$
& $0.794 \pm 0.004$
& {\bfseries\boldmath $0.814 \pm 0.007$}
& $0.801 \pm 0.007$
& $0.799 \pm 0.004$ \\
& $AUC^B$
& $0.720 \pm 0.001$
& $0.769 \pm 0.017$
& $0.771 \pm 0.018$
& $0.586 \pm 0.002$
& {\bfseries\boldmath $0.788 \pm 0.006$} \\

\midrule

\multirow{2}{*}{\textbf{Adult}}
& Down. AUC
& $0.838 \pm 0.006$
& $0.711 \pm 0.044$
& {\bfseries\boldmath $0.841 \pm 0.003$}
& $0.802 \pm 0.017$
& $0.779 \pm 0.016$ \\
& $AUC^B$
& {\bfseries\boldmath $0.813 \pm 0.004$}
& $0.677 \pm 0.041$
& $0.693 \pm 0.004$
& $0.570 \pm 0.006$
& $0.716 \pm 0.018$ \\

\bottomrule
\end{tabular}
}
\end{table}

\end{document}